\pdfoutput=1
\documentclass[letterpaper,11pt]{article}
\usepackage[in]{fullpage} 
\usepackage[utf8]{inputenc}
\usepackage[numbers]{natbib} 
\usepackage{graphicx}
\include{macros.sty}
\newcommand\nnfootnote[2]{%
  \begin{NoHyper}
  \renewcommand\thefootnote{#1}\footnotetext{#2}%
  \end{NoHyper}
}

\title{Ramsey Theorems for Trees and a
General\\ `Private Learning Implies Online Learning' Theorem}
\author{Simone Fioravanti$^{*}$ 
\and 
Steve Hanneke$^{\dag}$
\and 
Shay Moran$^{\ddag,\S, \P}$
\and
Hilla Schefler$^{\ddag}$
\and 
Iska Tsubari$^{\ddag}$
}

\nnfootnote{$*$}{Department of Computer Science, Gran Sasso Science Institute (GSSI)}
\nnfootnote{$\dag$}{Department of Computer Science, Purdue University}
\nnfootnote{$\ddag$}{Department of Mathematics, Technion}
\nnfootnote{$\S$}{Departments of Computer Science, \& Data and Decision Sciences, Technion}
\nnfootnote{$\P$}{Google Research}
\date{}

\DeclareUnicodeCharacter{2212}{\ensuremath{-}} 
\begin{document}
\maketitle
\thispagestyle{empty}
\begin{abstract}
This work continues to investigate the link between differentially private (DP) and online learning.~\citet*{AlonLMM19} showed that for binary concept classes, DP learnability of a given class implies that it has a finite Littlestone dimension (equivalently, that it is online learnable). 
Their proof relies on a model-theoretic result by \citet*{HodgesModelTheory}, which demonstrates that any binary concept class with a large Littlestone dimension contains a large subclass of thresholds. In a follow-up work, \citet*{JungKT20} extended this proof to multiclass PAC learning with a bounded number of labels. Unfortunately, Hodges's result does not apply in other natural settings such as multiclass PAC learning with an unbounded label space, and PAC learning of partial concept classes. 

This naturally raises the question of whether DP learnability continues to imply online learnability in more general scenarios: indeed,~\citet*{AlonHHM21} explicitly leave it as an open question in the context of partial concept classes, and the same question is open in the general multiclass setting.
In this work, we give a positive answer to these questions showing that for general classification tasks, DP learnability implies online learnability. Our proof reasons directly about Littlestone trees, without relying on thresholds. We achieve this by establishing several Ramsey-type theorems for trees, which might be of independent interest.
\end{abstract}

\newpage
\setcounter{page}{1}
\section{Introduction}\label{sec:intro}
\paragraph{Ramsey Theory.}
In theoretical research, it often happens that a general result, born to address a specific problem, proves to be of great interest by itself and inspires new original research. Indeed, Frank P. Ramsey clearly set the primary goal of his seminal work~\cite{Ramsey1930} to be a problem of formal logic and described his fundamental theorem in the abstract as a result of ``independent interest" useful in that specific setting. It was only a few years later, thanks to~\citet{ErdosSzekeres35}, that Ramsey's theorem, in its most well-known graph-theoretic variant, got the attention of the mathematical community, marking the beginning of what is now a well-established branch of mathematical research: Ramsey theory.

In a broad sense, Ramsey theory refers to a growing literature of results showing that, given a partition of a mathematical structure into classes, if it is large enough, then it will admit a sub-structure that is regular in some sense. The problem of quantifying what ``large enough" means is usually greatly challenging, and many existence results are not accompanied by quantitative bounds on the so-called \emph{Ramsey numbers}. As Erd\"os eloquently put it
\begin{quote}
    Suppose aliens invade the earth and threaten to obliterate it in a year's time unless human beings can find the Ramsey number for red five and blue five. \textelp{} within a year we could probably calculate the value. If the aliens demanded the Ramsey number for red six and blue six, however, we would have no choice but to launch a preemptive attack.
\end{quote}
Open problems in Ramsey theory are of great independent interest, but are also significant in combinatorics as a whole, since, as noted in~\cite{ConlonFS15}, they were crucial in the development of different theories (like random graphs and the probabilistic method~\cite{AlonBook}). 

Applications of Ramsey theory, however, are not limited to combinatorics and formal logic but extend to many areas of both mathematics and theoretical computer science: to name but a few, number theory~\cite{Szemeredi1975}, convex geometry~\cite{Barany2000}, ergodic theory~\cite{FurstenbergBook}, and lower bounds for computing boolean functions~\cite{AlonM86}.
A special mention goes to the quest for finding explicit constructions for graphs witnessing the tightness of Ramsey theorems: these results raise great interest in theoretical computer science because of their link with the theory of pseudorandomness. Some relevant works on the subject are~\cite{Eshan16,BarakKSSW10,Frankl1981IntersectionTW}, while we refer to the book by~\citet{AviWigBook} for an overview.
We refer the interested reader to the book~\cite{graham1991ramsey} for a detailed introduction to Ramsey theory and to~\cite{Rosta2004,WebRamseyCS,Roberts84} for a broader survey of possible applications.

\paragraph{Ramsey Theory and Differential Privacy.}
Recently, Ramsey theory has been employed to prove impossibility results in the differentially private PAC learning (or DP learning) model~\cite{KasiLNRS11}, merging the requirements of \emph{differential privacy}~\cite{DworkMNS06} with the ones of classical PAC learning~\cite{Valiant84}. Informally, a randomized learning algorithm, denoted by $\cA$, differentially-privately learns a concept class $\cH$ if (1) it is able to learn a target concept $c\in \cH$ from a limited number of input data, and also (2) given two training samples that differ by at most one example, the posterior distributions over the outputs of $\cA$ on the two of them are approximately the same.

In this context, it is natural to ask whether known PAC learnable classes can be learned by a DP algorithm (i.e.\ assessing the ``cost of privacy"). One of the most basic and best studied classes is the one of linear classifiers (a.k.a.\ threshold functions) in $\mathbb{R}^d$. 
While it is well known that this class is PAC learnable, a seminal result by~\citet*{BunNSV15} shows that it is indeed impossible to \underline{properly} DP learn it, by proving a lower bound on the sample complexity for one-dimensional thresholds.
Later, in his thesis~\citet{BunThesis} gave an alternative proof for this bound using a Ramsey-theoretic argument. 

A more recent work by~\citet*{AlonBLMM22} extends the same result to improper PAC learners, again using a Ramsey-theoretic argument. Specifically, their result applies more broadly to any class~$\cH$ of binary functions with unbounded \emph{Littlestone dimension} $\LD{\cH}$~\cite{Lit87,Ben-DavidPS09}. This combinatorial parameter, defined as the maximal depth of a so-called Littlestone tree, is known to control the online learnability of a class $\cH$. In particular, $\cH$ is online learnable if and only if $\LD{\cH} < \infty$.

Interestingly, \citet*{BunLM20} demonstrated that a finite Littlestone dimension is not only necessary but also sufficient for differentially private PAC learnability. This establishes an equivalence between online learning and differentially private PAC learning for binary concept classes.
It is only natural to ask if similar methodologies could be employed to extend these results to more general learning problems.

\paragraph{Online Learnability and Privacy: Beyond Binary Classes.}
While the application of Ramsey theory seems indeed promising, the proofs shown in~\cite{AlonBLMM22} significantly rely on the equivalence between $\LD{\cH}$ and the so-called \emph{Threshold Dimension} $\TD{\cH}$ which is very specific to binary classes. In other domains, in fact, even when it is possible to define a concept analogous to thresholds (which is often not the case), such an equivalence is not guaranteed to exist. 

An illuminating example comes from the domain of \emph{partial concept classes}.
These represent a generalization of the usual notion of concept classes $\cH$
(i.e., sets of binary functions $\cX \to \{0,1\}$)
to allow for \emph{partial} functions
$\cX \to \{0,1,*\}$
(where ``$*$'' denotes an \emph{undefined} value);
these are important for modeling within the PAC model natural data-dependent assumptions, such as learning halfspaces with a \emph{margin}. They were first considered by \citet*{LongPartial} and later developed further by \citet*{AlonHHM21}.
Although some known principles holding for total concepts remain valid for learning partial concepts (such as the equivalence between PAC learnability and finite VC dimension), there are significant differences. In particular, \citet{AlonHHM21} prove, in fact, the existence of a binary partial concept class $\cH$ having infinite $\LD{\cH}$ but $\TD{\cH}=2$, as described in more detail in~\Cref{sec:technical:privacy}. Consequently, \citet{AlonHHM21} leave as an open question whether the relationship between online learning and private PAC learning extends to partial concept classes.

A first extension of the result of~\citet*{AlonBLMM22} showing equivalence of DP learnability to online learnability has been shown in~\cite{JungKT20,SivakumarBG21} for multi-valued concept classes (or multiclasses): i.e., classes of total concepts for which the label space~$\cY$ is not restricted to be $\{0,1\}$. In particular, \citet*{JungKT20} adapt the techniques used in~\cite{AlonBLMM22} to prove that the equivalence continues to hold in the case of a finite $\cY$.
However, such techniques break down in the case of infinite label spaces~$\cY$, where in general the connection between $\LD{\cH}$ and threshold-like notions no longer holds (see \cite{hanneke:23}).

Studying the relationship between private PAC learnability and online learnability in general settings, such as partial concept classes or multiclasses with infinite label spaces $\cY$, appears to require techniques and results allowing to work on objects (like Littlestone trees) that are more structured than simple thresholds.
This in turn requires us to extend quantitative results in Ramsey theory to such structures as well. 

\paragraph{Our Contribution.}
In this work, our contribution is twofold. 
First, we prove several Ramsey theorems in the context of trees. In particular, we introduce a new notion of \emph{type} for subsets of vertices of a tree, and show its importance to build a flexible Ramsey theory for trees. By differentiating subsets of vertices of a binary tree according to our definition of type, we are able to show general existence results as well as quantitative ones in specific cases.

Second, we extend the impossibility result holding for binary DP learning to the more general setting of \underline{partial} concept classes with \underline{infinite} label space~$\cY$. With the aid of our Ramsey theorems on trees, we are able to argue directly about Littlestone trees without passing through the threshold dimension. In particular, our results address an open question in~\cite{AlonHHM21}, by proving one direction of the qualitative equivalence between DP and online learning for (binary) partial concept classes. It remains open to show whether a finite Littlestone dimension implies DP-learnability in this case.

Moreover, we contribute to the understanding of this equivalence in the context of total multiclasses, previously studied in~\cite{JungKT20,SivakumarBG21}, by showing that DP-learnability implies online learnability also in the case of an infinite label space $\cY$. In this setting as well, it remains an open question whether the opposite direction of the equivalence continues to hold.

\paragraph{Organization.} 
The paper is organized as follows. The rest of this section provides a detailed exposition of our main results in both Ramsey theory and privacy and discusses related works. \Cref{sec:technical} highlights the technical aspects of our main proofs. \Cref{sec:preliminaries} offers preliminary information and relevant background. Lastly, \Cref{sec:privacy_vs_LD,sec:ramsey_theory_in_trees} present the full proofs of our results.

\subsection{Main Results}\label{sec:main_results}
This section presents the main technical contributions of the paper, focusing on the ideas and intuitions behind them. In particular,~\Cref{sec:main_results:ramsey} introduces our novel Ramsey-theoretic results in the context of trees, in increasing order of complexity.~\Cref{sec:main_results:privacy} in turn focuses on their application to private PAC learnability in the setting of partial concept classes and multiclass classification.

\subsubsection{Ramsey Theory for Trees}\label{sec:main_results:ramsey}
Ramsey stated his theorem in~\cite{Ramsey1930} for subsets of a given set. In its finite version, the theorem states the following: \emph{for all $d, m, k \in \bbN$, with $m \ge 2$, there always exists a large enough number $n=n(d,m,k) \in \bbN$ such that, given any set of size $n$, however we color its subsets of size $m$ with~$k$ colors, it admits a monochromatic subset of size $d$}.
Ramsey's theorem can be considered as a generalization of the pigeonhole principle, which corresponds to the case $m=1$ and constitutes a building block in the proof of the general theorem.

In our work, the main structure of interest is the complete binary tree $T$ of a certain depth $n$. Here, we consider to be a subtree of depth $d$ any collection of $2^{d+1}-1$ vertices that is isomorphic to the complete binary tree of depth $d$, in the sense that the left/right descendant relations are preserved. ~\Cref{subfig:subtree} illustrates this concept. As for the classical Ramsey theory, the first necessary step to build a solid Ramsey theory for trees is to prove the equivalent of a pigeonhole principle: i.e.,\ a result showing that, given $d\in \bbN$ and any coloring of a sufficiently deep tree $T$ with $k$ colors, there exists a monochromatic subtree $S$ of $T$ with depth $d$. 
Such results already exist in the literature. In particular, \citet*{HodgesModelTheory} show a version with our same definition of subtree, while both~\citet{FurstenbergWeiss2003} and~\citet*{PachTS12} prove a quantitatively weaker version holding for a more restrictive definition of subtree. See~\Cref{prop:php} for a formal statement of both versions.

\begin{figure}[htb]
\centering
\begin{subfigure}[l]{0.6\textwidth}
\centering
\begin{adjustbox}{valign=t}
\begin{forest}
rounded/.style={circle, draw, fill opacity=0.4}
    [{},fill=red, for tree=rounded
      [{},
        [{},]
        [{},fill=red]
      ]
      [{},fill=red
        [{},]
        [{},]
      ]
    ]
\end{forest} \hspace{0.2cm}
\end{adjustbox}
\begin{adjustbox}{valign=t}
\begin{forest}
rounded/.style={circle, draw, fill opacity=0.4}
    [{},fill=black, for tree=rounded
      [{},fill=black
        [{},fill=black]
        [{},]
      ]
      [{},
        [{},]
        [{},]
      ]
    ]
\end{forest} \hspace{0.2cm}
\end{adjustbox}
\begin{adjustbox}{valign=t}
\begin{forest}
rounded/.style={circle, draw, fill opacity=0.4}
    [{}, fill=black, for tree=rounded
      [{},[{},fill=black][{},fill=black]]
      [{},[{},][{},]]
    ]
\end{forest}
\end{adjustbox}
\caption{}
\label{subfig:subtree}
\end{subfigure}\hspace{1cm}
\begin{subfigure}[r]{0.2\textwidth}
    \centering
    \begin{adjustbox}{valign=t}
\begin{forest}
rounded/.style={circle, draw, fill opacity=0.4}
    [{}, name=root, for tree=rounded
      [{}, name=l,
      [{}, name=ll,]
      [{}, name=lr,]
      ]
      [{}, name=r 
      [{}, name=rl,]
      [{}, name=rr,]]
    ]
 \path let
         \p1 = ($(root.north west)-(r.south east)$),
         \n1 = {veclen(\p1)}
         in
         (root.south east) -- (r.north west) 
         node[midway, sloped, draw, ellipse, thick, blue, minimum width=\n1*1.04, minimum height=\n1/2.5, rotate=5] {};
\path let
         \p2 = ($(root.north east)-(lr.south west)$),
         \n2 = {veclen(\p2)}
         in
         (root.south west) -- (lr.north east) 
         node[midway, sloped, draw, ellipse, thick, red, minimum width=\n2*1.05, minimum height=\n2/3.5, rotate=-7] {};
         \node[draw,green, ellipse, thick, xscale=0.55, yscale=0.7, fit = {(lr) (rl)}]{};
\end{forest}
\end{adjustbox}
\caption{}
\label{subfig:counterexample}
\end{subfigure}
\caption{In (\protect\subref{subfig:subtree}) the colored vertices denote a candidate $S$ for a subtree. According to our definition, the $S$ colored in red is the only legal subtree. 
The example in (\protect\subref{subfig:counterexample}) shows a coloring of pairs (coloring \emph{left} relations, \emph{right} relations, and \emph{incomparable} relations in different colors) for which no monochromatic subtree can exist.
}
\end{figure}
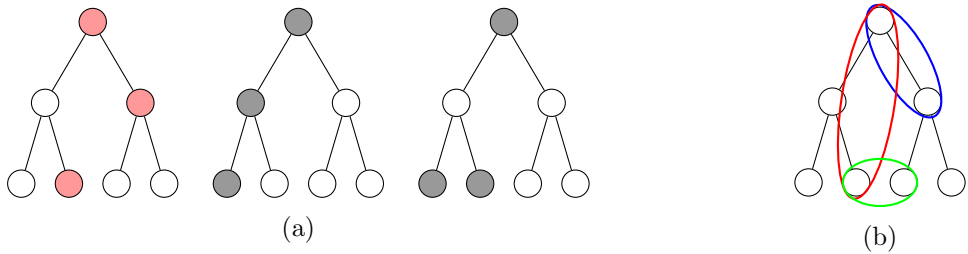

\paragraph{Case of Pairs ($m=2$).}
It is natural to ask if the existence of monochromatic subtrees is possible if one considers colorings of subsets of vertices of a given size, as in the classical Ramsey's theorem. We argue that a na\"{i}ve extension like this one is not possible in the context of trees.
Let us consider the simplest possible case of coloring pairs of vertices of a complete binary tree $T$ of arbitrary depth. Consider coloring the pairs following the pattern depicted in~\Cref{subfig:counterexample}; this example admits neither a finite nor an infinite Ramsey theorem. The example works as follows.
Observe that, given any pair of vertices $\{u,v\}$ of $T$, either $u$ and $v$ are \emph{comparable} in the partial order induced by the tree structure or they are \emph{incomparable}.
Define a coloring $\chi$ on pairs of vertices as follows: $\chi(\{u,v\}) =\mathsf{red}$ if one of $u$ or $v$ is a left descendant of the other, $\chi(\{u,v\}) =\mathsf{blue}$ if one of $u$ or $v$ is a right descendant of the other, and $\chi(\{u,v\}) =\mathsf{green}$ if $u,v$ are incomparable.
It is straightforward to see that, if we color all pairs in a complete tree (finite or infinite) in this way, any subtree of depth $\ge 1$ cannot be monochromatic (in the sense of the original Ramsey theorem), or even dichromatic (i.e.\ admitting at most $2$ colors), since it will contain pairs colored $\mathsf{red}$, $\mathsf{blue}$, and $\mathsf{green}$.
Let us say that a pair has \emph{type} left/right, if one of the vertices is a left/right descendant of the other, and incomparable otherwise.

Our first theorem shows that, if we consider subtrees admitting pairs colored with at most $3$ colors (which we call \emph{trichromatic}) we are able to prove existence:

\begin{boxH}
\begin{maintheorem}[Ramsey for pairs]\label{thm:ramsey_pairs}
    For all $d,k\in \bbN$ there exists $n=n(d,k)$ such that for every coloring of all pairs of vertices in the complete binary tree of depth $n$ with $k$ colors, there exists a trichromatic complete subtree of depth $d$.

    Furthermore, the obtained subtree is \emph{type-monochromatic}, in the sense that if two pairs are of the same type then they are colored with the same color.
\end{maintheorem}
\end{boxH}
The proof of \Cref{thm:ramsey_pairs} is given in \Cref{sec:ramsey_pairs_proof}.

\begin{remark}[Quantitative Bounds for Pairs] 
  \Cref{thm:ramsey_pairs} states the existence of a trichromatic subtree with the additional property of being type-monochromatic, i.e., pairs of the same type are colored with the same color. As a direct consequence, if we color only pairs of a specific type (say only pairs of incomparable vertices), we obtain the existence of monochromatic subtrees. Moreover, our proof implies upper bounds on the minimal $n(d,k)$ witnessing \Cref{thm:ramsey_pairs}: in particular, if we color all pairs, we obtain a bound of $n(d,k) \le \twr_{(8dk)}(1)$, where $\twr_{(m)}$ denotes the \emph{tower} function (i.e., iterated exponentiation $m$ times; see \Cref{sec:additional_defs}). 

  If we color only pairs of a specific type, however, we get better bounds that differ greatly based on the type. Specifically, if the pairs are comparable (i.e., the chosen type is left or right), we obtain an upper bound of $2^{dk\log(2k)}$ (see \cref{prop:ramsey_pairs_comp}). This bound is comparable to the known bounds for classical hypergraph Ramsey numbers (see \Cref{thm:ramsey_chains}). 

  On the other hand, if we consider only incomparable pairs, we obtain a bound of $n(d,k) \le \twr_{(7dk)}(1)$, which is comparable to the bound holding for colorings of all pairs. The quantitative results described are summed up in \Cref{table:pairs}.
\end{remark} 

\begin{table}
\begin{center}
\begin{tabular}{ |c|c|c|c| }
\hline
All pairs & Left & Right & Incomparable \\
\hline
$\left(\twr_{(3dk+\log^{\star} 4k +4)}(1)\right)^{5k^2\log k}\le \twr_{(8dk)}(1)$& $2^{dk\log(2k)}$ & $2^{dk\log(2k)}$ & $\twr_{(3dk+\log^{\star} 4k +3)}(1)\le \twr_{(7dk)}(1)$\\
\hline
\end{tabular}
\end{center}
\caption{The quantitative upper bounds described in Remark 1 for pairs. Each column reports the bound holding when coloring only pairs of the according type(s).}
\label{table:pairs}
\end{table}

It is an interesting open question to assess whether the bounds shown in the remark could be improved.

\paragraph{General Case ($m \ge 2$).}
Next, we treat the more general case where the coloring is of subsets of vertices of size $m\ge 2$ (or $m$-subsets).

The first necessary step, as the case of pairs well illustrates, is to define a notion encoding the partial order inherited from $T$ as well as differentiate between left and right descendants and incomparable vertices. A first natural attempt at defining such a concept could be to say that two $m$-subsets are of the same type if there is a bijection from one set to the other that preserves the left/right descendant relations between any pair of vertices.
However, this approach immediately fails as demonstrated in~\Cref{fig:general_type}. This example shows that one must consider also left/right relations between vertices that are not comparable with respect to the partial order of descendancy.

\begin{figure}[!htb]
\centering
\begin{subfigure}[l]{0.14\textwidth}
\centering
\begin{adjustbox}{valign=t}
\begin{forest}
rounded/.style={circle, draw, fill opacity=0.4}
    [{}, for tree=rounded
      [{}, fill
        [{},fill]
        [{},]
      ]
      [{},fill
        [{},]
        [{},]
      ]
    ]
\end{forest} \hspace{1cm}
\end{adjustbox}
\caption{\textcolor{red}{red} triplet}
\label{subfig:general_type_a}
\end{subfigure}\hspace{1cm}
\begin{subfigure}[r]{0.14\textwidth}
\centering
\begin{adjustbox}{valign=t}
\begin{forest}
rounded/.style={circle, draw, fill opacity=0.4}
    [{}, for tree=rounded
      [{},fill
        [{},]
        [{},]
      ]
      [{},fill
        [{},fill]
        [{},]
      ]
    ]
\end{forest} \hspace{1cm}
\end{adjustbox}
\caption{\textcolor{blue}{blue} triplet}
\label{subfig:general_type_b}
\end{subfigure}
\caption{This example shows two triplets of vertices that admit the same relations of left/right descendant for each pair of vertices, for appropriate ordering of the vertices. However, consider coloring \textcolor{red}{red} all triplets of the form as in (\protect\subref{subfig:general_type_a})
(the single vertex is to the right of the chain of size two), and coloring \textcolor{blue}{blue} all triplets of the form as in (\protect\subref{subfig:general_type_b}) (the single vertex is to the left of the chain of size two). Then within any tree (finite or infinite) with every triplet of these two kinds colored in this way, there is no monochromatic subtree of depth $2$, as these two patterns must appear in any such subtree.
}
\label{fig:general_type}
\end{figure}
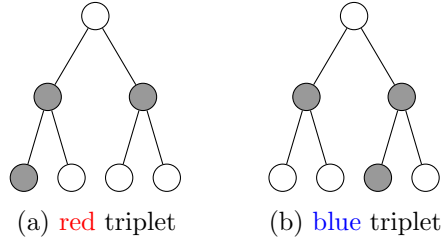

Our first contribution for the case of general $m$ is an appropriate formal definition of ``type'' for sets of vertices of size $m$ which avoids the above problem. 
Before stating this definition though, we need to introduce the concept of \emph{closure}.
\begin{definition} [Closure]\label{def:closure}
    Let $T$ be a complete binary tree and $A$ a set of vertices of $T$. We say that~$A$ is closed if for each $u,v\in A$, their lowest-common-ancestor $\LCA(u,v)$ is also in $A$. The \emph{closure}~$\bar{A}$ of $A$ in $T$ is the minimal\footnote{Note that the intersection of closed sets is closed, and hence any set $A$ has a unique closure.} closed set that contains $A$ (minimal with respect to set containment).
\end{definition}

\begin{figure}[thb]
\centering
\begin{adjustbox}{valign=t}
\begin{forest}
rounded/.style={circle, draw, fill opacity=0.5}
    [{}, red, for tree=rounded
      [{}, red, [{},fill=red] [{},fill=red]]
      [{}, [{},][{},fill=red]]
    ]
\end{forest}\hspace{1cm}
\end{adjustbox}
\begin{adjustbox}{valign=t}
\begin{forest}
rounded/.style={circle, draw, fill opacity=0.4}
    [{}, red, for tree=rounded
      [{},fill=red [{},fill=red] [{},]]
      [{},fill=red, [{},] [{},]]
    ]
\end{forest}\hspace{1cm}
\end{adjustbox}
\begin{adjustbox}{valign=t}
\begin{forest}
rounded/.style={circle, draw, fill opacity=0.4}
    [{}, red, for tree=rounded
      [{},fill=red [{},] [{},]]
      [{},fill=red,fill=red, [{},fill=red][{},]]
    ]
\end{forest}
\end{adjustbox}
\caption{In each of the examples, the vertices colored red represent $A$, while the vertices circled in red are the ones in $\bar{A}\setminus A$. According to~\Cref{def:type_subset}, each example represents a subset of $3$ vertices of a different type.}
\label{fig:closures}
\end{figure}
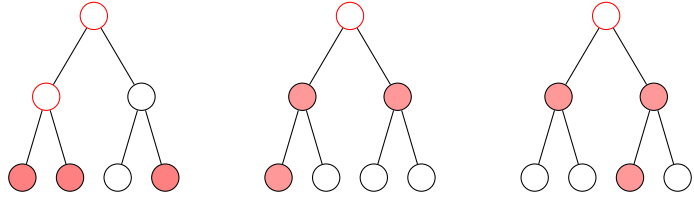

Different examples of sets of vertices along with their closures, are shown in~\Cref{fig:closures}. Intuitively, considering~$\bar{A}$ instead of $A$ allows us to differentiate between sets like the middle and rightmost in~\Cref{fig:closures}. In fact, these two sets of vertices are symmetrical and the relations among each pair are the same; however, if we consider their closures, the left (right) descendant relations are clearly different.
The following definition formalizes this exact reasoning:
\begin{definition} [Types]\label{def:type_subset}
    Let $T$ be a complete binary tree and $A_1, A_2$ two distinct subsets of vertices of $T$. We say that $A_1$ is \emph{isomorphic} to $A_2$, and denote it by $A_1 \simeq A_2$, if there is a bijection $\varphi$ from $\bar{A_1}$ to $\bar{A_2}$ such that $\varphi(A_1)=A_2$ and $\varphi$ preserves the left (right) descendant relations between pairs of vertices.
    The \emph{type} of each set of vertices is then its induced equivalence class, and we say that sets are of the same type if they are isomorphic to one another.
\end{definition}

Our main result in this context employs the notion of type to effectively generalize \Cref{thm:ramsey_pairs} to $m$-subsets, as summarized in the following:

\begin{boxH}
\begin{maintheorem}[Ramsey for $m$-subsets]\label{thm:ramsey_general_finite}
    Let us denote by $\tau(m)$ the number of possible types of $m$-subsets.
    Then, for all $d, m, k\in \bbN$ there exists $n=n(d,m,k)$ such that, for every coloring of all $m$-subsets of the complete binary tree of depth $n$ with $k$ colors, there exists a $\tau(m)$-chromatic subtree of depth $d$ (i.e.\ its $m$-subsets are colored with at most $\tau(m)$ colors).

    Furthermore, the obtained subtree is \emph{type-monochromatic}, {in the sense that if two $m$-subsets are of the same type then they are colored with the same color.}
\end{maintheorem}
\end{boxH}
The proof of ~\Cref{thm:ramsey_general_finite}
appears in \Cref{sec:ramsey_general_proof}.

\begin{remark}[Optimal Number of Colors]\label{remark:optimal_num_of_clrs}
    Note that $\tau(m)$ colors are optimal in the following sense: if we color $m$-subsets  of the complete infinite subtree with $\tau(m)$ colors according to their type, then every complete subtree with depth $ \ge m$ will contain at least one $m$-subset colored with each color. Moreover, observe that $\tau(m)$ depends exclusively on $m$, i.e.\ the size of the chosen subsets, and not on $k$ and $d$. 
    In~\Cref{app:types}, we show that $\tau(m)$ is upper bounded by {$\frac{2^{3m-2}}{\sqrt{\pi\cdot(m-1)}}$}.
\end{remark}

\begin{remark}[Infinite Case]
By compactness, \cref{thm:ramsey_general_finite} is equivalent to the following statement: for every coloring of all $m$-subsets of an infinite complete binary tree\footnote{An infinite complete binary tree is a tree such that every vertex has exactly two children.}, there exists a $\tau(m)$-chromatic subtree of arbitrarily large finite depth. However, it is not true that there exists an \underline{infinite} $\tau(m)$-chromatic subtree.

We discuss the infinite case in \cref{app:ramsey_infinite}. In particular, we show that the number of colors~$\tau_{\infty}(m)$ needed to guarantee the existence of an infinite $\tau_{\infty}(m)$-chromatic subtree is strictly larger than $\tau(m)$; for example, for pairs we have $\tau_{\infty}(2)=4 > 3=\tau(2)$. In the general case, we show that $\tau_{\infty}(m)=\omega(\tau(m))$, meaning the optimal number of colors in the infinite case is asymptotically larger than in the finite case.
\end{remark}
Our proof of \Cref{thm:ramsey_general_finite} is qualitative and does not provide explicit bounds on the Ramsey numbers. However, if we focus on $m$-chains (i.e., $m$-subsets totally ordered by the tree structure), we can derive quantitative bounds. These results, summarized below, are crucial to our main result in private PAC learning.
\begin{boxH}
\begin{maintheorem}[Ramsey for $m$-chains]\label{thm:ramsey_chains}
 For all $d,k,m$ there exists \[
 n \le \twr_{(m)}(5\cdot 2^{m-2}dk^{2^{m-1}}\log k)
 \]
 such that, for every coloring of $m$-chains in the complete binary tree of depth $n$, there exists a $2^{m-1}$-chromatic complete subtree of depth $d$ (i.e.\ its $m$-chains are colored with at most $2^{m-1}$ colors). 

 Furthermore, the obtained subtree is \emph{type-monochromatic}, in the sense that if two $m$-chains are of the same type then they are colored with the same color.
\end{maintheorem}
\end{boxH}

The proof of \Cref{thm:ramsey_chains} can be found in \Cref{sec:ramsey_chains_proof}.
Note that the number of possible types of $m$-chains is exactly $2^{m-1}$. Indeed, since the closure of an $m$-chain is exactly the $m$-chain itself, the type encodes the left/right descendant relations between all pairs of vertices in the chain. In particular, since the left/right relations among pairs of consecutive vertices in the chain are sufficient to fully describe all the others, we get a total of $2^{m-1}$ types.
As a consequence, Remark~\ref{remark:optimal_num_of_clrs} made above for \Cref{thm:ramsey_general_finite} continues to hold in this case.

\begin{remark}[Tightness of the Bound]
Regarding the upper bound on the Ramsey number $n$ provided above, we do not know whether it is tight. However, if we restrict the coloring to $m$-chains of a specific type, the bound significantly improves. Specifically, we obtain:
\begin{equation}\label{eq:boundchains}
n \le \twr_{(m)}(5 \cdot 2^{m-2} dk \log k).    
\end{equation}
This removes the $2^{m-1}$ from the exponent of $k$, making the bound comparable to the best known upper bounds on hypergraph Ramsey numbers by~\citet{ErdosRado52}. Conversely, any lower bound for classic Ramsey numbers on coloring $m$-subsets holds for $m$-chains of a specific type $\tau$: indeed, given a coloring $\chi$ of the $m$-subsets of $[n]$, define a coloring $\chi'$ of $m$-chains of type $\tau$ with the color~$\chi$ assigns to their set of levels. The obtained coloring $\chi'$ admits complete monochromatic subtrees of depth $d$ if and only if $\chi$ admits monochromatic subsets of size $d$.
\end{remark}

\subsubsection{Private PAC Learning Implies Finite Littlestone Dimension}\label{sec:main_results:privacy}
In this section, we present our results concerning differentially private PAC learnability (from now on DP-learnability) of partial multiclasses $\cH$: i.e.,\ partial concept classes where the label space $\cY$ is not constrained to be $\{0,1\}$. Throughout the discussion, we use some standard technical terms and notation: for formal definitions, we refer the reader to~\Cref{sec:preliminaries_learning}.
The following theorem shows that the sample complexity of DP learning a partial multiclass $\cH$ is lower bounded by a function of its Littlestone dimension:

\begin{boxH}
    
\begin{maintheorem}\label{thm:DP_implies_LD}
Let $\cH$ be a (possibly partial) concept class over an arbitrary label space $\cY$ with Littlestone dimension ${\LD{\cH}\geq d}$, 
and let $\cA$ be an $(10^{-4},10^{-4})$-accurate learning algorithm for $\cH$ with sample complexity $m$ satisfying $(\epsilon,\delta)$-differential privacy for $\epsilon=10^{-3}$ and ${\delta\leq \frac{1}{10^3m^2}}$.
Then the following bound holds:
\begin{equation*}
   m = \Omega(\log^\star d), 
\end{equation*}
where the $\Omega$ notation conceals a universal numerical multiplicative constant.
\end{maintheorem}
\end{boxH}
\Cref{thm:DP_implies_LD} extends the result by~\cite{AlonLMM19}, who handled total concept classes in the binary case, and the result by~\cite{JungKT20}, who handled total concept classes in the multiclass setting with $k$ labels (deriving a lower bound of $\log^\star \frac{\log_k d}{k^2}$, which depends on $k$).
The proof of \Cref{thm:DP_implies_LD} can be found in \Cref{sec:privacy_vs_LD}. An immediate consequence of this lower bound is that it is impossible to privately PAC learn a class $\cH$ with an infinite Littlestone dimension.

\begin{boxH}
 \begin{maincorollary}
[DP-Learnability$\implies$Finite Littlestone Dimension]\label{cor:DP_implies_LD}
Let $\cH$ be a partial concept class over an arbitrary label space $\cY$. If $\cH$ is DP-learnable, then it has a finite Littlestone dimension, $\LD{\cH} < \infty$. Equivalently, if $\cH$ is DP-learnable, then it is also online learnable.
\end{maincorollary}
\end{boxH}

\citet*{AlonHHM21} studied PAC learnability of partial concept classes in the binary case. They asked whether the characterization of differentially private PAC learnability by the Littlestone and threshold dimensions extends to this setting. They further exhibited a class with an unbounded Littlestone dimension and a threshold dimension of 2, which shows that the equivalence between threshold and Littlestone dimensions breaks for partial concept classes. Combined with the above result, this implies that the threshold dimension does not characterize private PAC learnability:

\begin{boxH}
 \begin{maincorollary}
[DP-Learnability $\neq$ Finite Threshold Dimension]\label{cor:DP_vs_TD}
There exists a partial concept class $\cH \subseteq \{0,1,\star\}^\cX$ such that (i) $\cH$ has a finite threshold dimension, and (ii) $\cH$ is not differentially privately PAC learnable.
\end{maincorollary}
\end{boxH}

It remains open whether a finite Littlestone dimension characterizes private PAC learnability (\Cref{cor:DP_implies_LD} provides one direction; the other direction is yet to be decided).

Moreover, our results close a gap regarding the equivalence between online and DP-learnability of (total) multi-valued concept classes. In fact, as already mentioned previously, the results in~\cite{JungKT20,SivakumarBG21} showed that the equivalence continues to hold in the case of a finite label space~$\cY$. 
However, their proof techniques do not extend 
to infinite label spaces (where the connection between Littlestone dimension and threshold-like structures no longer holds \cite{hanneke:23}).
Our results show that DP-learnability continues to imply online learnability in the case of an infinite~$\cY$, leaving as an open question whether the converse holds as well.

\subsection{Additional Related Work}\label{sec:related_work}
\paragraph{Ramsey Theory on Trees.} 
In general, a Ramsey theorem for trees is a statement of the following form: for any coloring of $m$-subsets of a sufficiently deep tree, there exists a deep subtree that is monochromatic (or has only a few colors). Various forms of these statements have been explored, differing mainly in how a subtree is defined. In particular, the following variations have been considered (see Figure~\ref{fig:subtree_defs}):

\begin{enumerate}
    \item[(i)] The simplest definition of a subtree preserves only the ``descendant of'' relation.  That is, a complete subtree of depth $d$ of a binary tree $T$ consists of any collection of $2^{d+1}-1$ vertices in $T$ that maintain the same ``descendant of'' relationships as in a complete binary tree of depth $d$.

    \item[(ii)] The definition we consider in this paper preserves also the relationships of left/right descendants. Thus, every subtree in this case is also a subtree as defined in the first case, but not vice versa.

    \item[(iii)] The most studied definition also accounts for levels: on top of (ii), a subtree here must ensure that any pair of vertices on the same level in the subtree is also on the same level in $T$.
\end{enumerate}

Milliken's theorem~\cite{Milliken79} is arguably the most celebrated Ramsey-type result for trees and is stated with respect to Item~(iii) above. We call subtrees as the ones in Milliken's theorem \emph{level-aligned} and refer the reader to~\Cref{sec:preliminaries_combi} for more details. The theorem states the following:
if $T$ is a tree of infinite depth, where each vertex has a finite non-zero number of children, and we color level-aligned subtrees of depth~$m$ with $k$ colors, then there exists a monochromatic level-aligned subtree~$S$ of infinite depth. Milliken also proved a finite version of the theorem, but without providing any quantitative results; the strategy used in his proof makes it challenging to derive non-trivial bounds for the associated Ramsey numbers. Moreover, Ramsey's original result on sets can be derived as a corollary of Milliken's theorem.
In their recent monograph, \citet*{2024milliken} explore applications of Milliken’s tree theorem from the point of view of computability theory.
They focus on infinite level-aligned subtrees.
In particular, they prove that for every coloring of $m$-subsets in an infinite complete binary tree, there exists an infinite level-aligned subtree colored with at most $\tau_{\infty}^{lvl}(m)$ colors, and that the number of colors $\tau_{\infty}^{lvl}(m)$ is optimal.
A Ramsey result regarding infinite trees within our setting (i.e. Item~(ii) above) appears in \cref{app:ramsey_infinite}.

A recent line of research furthered the study of the pigeonhole principle concerning level-aligned subtrees, initiated by~\citet*{FurstenbergWeiss2003}. In their work, the authors extend Furstenberg's ergodic-theoretic proof of Szemerédi's theorem on arithmetic progressions~\cite{Furstenberg1977, Szemeredi1975} to prove a pigeonhole principle for binary trees. Specifically, they show that, given a depth $d$ and number of colors $k$, there exists an $n = n(d, k)$ such that in any coloring of the vertices of a complete binary tree of depth $n$ using $k$ colors, there exists a level-aligned monochromatic subtree $S$ whose levels form an arithmetic progression in $T$. Notably, their argument did not provide any quantitative bounds on $n$. In a follow-up work, \citet*{PachTS12} gave a constructive proof of the pigeonhole principle within Milliken's setting (Item~(iii) above), where the levels of $S$ are not required to form an arithmetic progression in $T$. Specifically, they showed that, in this case, $n = \Theta(dk \log k)$ is sufficient and necessary. We use this result, which is formally stated in~\Cref{prop:php}.

\begin{figure}[!htb]
\centering
\begin{subfigure}{0.14\textwidth}
\centering
\begin{adjustbox}{valign=t}
\begin{forest}
rounded/.style={circle, draw, fill opacity=0.4}
    [{}, fill, for tree=rounded
      [{},
        [{},fill]
        [{},fill]
      ]
      [{},
        [{},]
        [{},]
      ]
    ]
\end{forest} 
\end{adjustbox}
\caption{}
\label{subfig:subtree_def_1}
\end{subfigure} \hspace{1cm}
\begin{subfigure}{0.14\textwidth}
\centering
\begin{adjustbox}{valign=t}
\begin{forest}
rounded/.style={circle, draw, fill opacity=0.4}
    [{}, fill, for tree=rounded
      [{},
        [{},fill]
        [{},]
      ]
      [{},fill
        [{},]
        [{},]
      ]
    ]
\end{forest} \hspace{1cm}
\end{adjustbox}
\caption{}
\label{subfig:subtree_def_2}
\end{subfigure}\hspace{1cm}
\begin{subfigure}{0.14\textwidth}
\centering
\begin{adjustbox}{valign=t}
\begin{forest}
rounded/.style={circle, draw, fill opacity=0.4}
    [{}, fill, for tree=rounded
      [{},
        [{},fill]
        [{},]
      ]
      [{},
        [{},fill]
        [{},]
      ]
    ]
\end{forest} \hspace{1cm}
\end{adjustbox}
\caption{}
\label{subfig:subtree_def_3}
\end{subfigure}\hspace{1cm}

\caption{These three examples summarize the definitions of a subtree discussed in Section \ref{sec:related_work}, ordered by their restrictiveness. sub-Figure (\protect\subref{subfig:subtree_def_1}) shows a valid subtree only according to the most relaxed definition, while sub-Figure (\protect\subref{subfig:subtree_def_3}) shows a valid subtree according to the most restrictive definition (and thus according to all other definitions as well).
}
\label{fig:subtree_defs}
\end{figure}
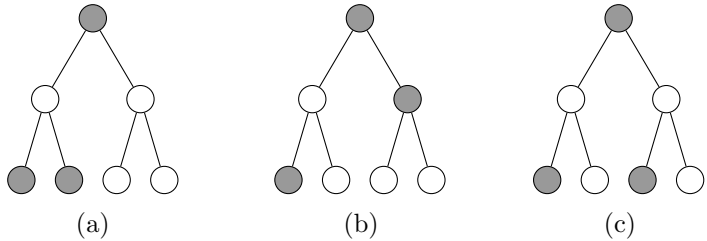

\paragraph{PAC Learning and Differential Privacy.}
PAC learnability of partial concept classes was initially studied in the binary setting in \cite{LongPartial} and \cite{AlonHHM21} and further explored in~\cite{Ali23}. In particular, ~\citet*{AlonHHM21} give a thorough overview of their PAC learnability in general, highlighting the differences with total concept classes as well as relevant open questions. Subsequent works have extended portions of that work to the online learning framework~\cite{CheungHHH23} and the multiclass setting \cite{brukhim2022characterization,KalavasisVK22,chirag:24}.

As mentioned in the Introduction, \citet*{JungKT20} show that the equivalence of DP learnability and online learnability continues to hold for multi-valued concept classes. In particular, they define a concept of thresholds allowing them to adapt the techniques used in~\cite{AlonBLMM22}. \citet*{SivakumarBG21} significantly improve the lower bound shown in~\cite{JungKT20} to prove one direction of the equivalence i.e.\ online learnability implies DP learnability.

Furthermore, extensive research has also been devoted to studying the question of which learning tasks can be performed subject to \emph{pure} differential privacy. Beimel et al. \citep{beimel2013characterizing,Beimel19Pure} introduced a quantity called \emph{representation dimension} that characterizes pure DP learnability. In a follow-up work, Feldman and Xiao \citep{FeldmanX15} found an interesting connection with communication complexity by associating every concept class $\cH$ with a communication task whose complexity characterizes whether $\cH$ is pure DP learnable. Additionally, \citet*{alon2023unified} showed a unified characterization for both pure and approximate differential privacy using cliques and fractional cliques of a certain graph corresponding to $\cH$.

\section{Proof Overview and Technical Highlights} \label{sec:technical}
In this section, we provide an outline of the proof ideas of the main theorems in this paper.

\subsection{Privacy}\label{sec:technical:privacy}

The proof of \Cref{thm:DP_implies_LD} consists of two parts: consider a DP algorithm as in the premise of \Cref{thm:DP_implies_LD}; the first part of the proof applies our developed Ramsey theorem for trees to identify a large subtree where the algorithm behaves in a `regular' manner. The second part shows that such algorithms can privately solve the interior point problem. This allows us to apply lower bounds for the interior point problem and conclude the stated lower bound.

The first part is similar to the first step in the proof by \citet*{AlonBLMM22} for the binary case, 
with the crucial difference that we use the Ramsey theorem for trees that we developed.
In contrast, the proof by \citet{AlonBLMM22} relies on the relationship between the Littlestone dimension and the threshold dimension of binary classes.
When considering partial concept classes and multiclasses over infinite label spaces, this connection fails to hold, as demonstrated in \Cref{fig:class_with_td=2_ld=infty}.
As a consequence, we needed to develop techniques allowing us to work directly with Littlestone trees, which are more structured than simple thresholds.

\begin{figure}[!htb]
\centering
\begin{subfigure}[l]{0.45\textwidth}
\centering
\begin{adjustbox}{valign=t}
\begin{forest}
rounded/.style={circle, draw, fill opacity=0.2}
    [{$x_1$}, for tree=rounded
      [{$x_2$}
            [{$x_4$}
                [{$h_0$},rectangle,rounded corners,fill,draw]
                [{$h_1$},rectangle,rounded corners,fill,draw]
            ]
            [{$x_5$}
                [{$h_2$},rectangle,rounded corners,fill,draw]
                [{$h_3$},rectangle,rounded corners,fill,draw]
            ]  
      ]
      [{$x_3$}
        [{$x_6$},                                   [{$h_4$},rectangle,rounded               corners,fill,draw]                          [{$h_5$},rectangle,rounded               corners,fill,draw] 
        ]
        [{$x_7$}
            [{$h_6$},rectangle,rounded               corners,fill,draw]                  [{$h_7$},rectangle,rounded               corners,fill,draw] 
        ]
      ]
    ]
\end{forest} \hspace{0.1\textwidth}
\end{adjustbox}
\end{subfigure}
\begin{subfigure}[r]{0.45\textwidth}
\centering
\begin{adjustbox}{valign=t}
\small
\begin{tabular}{|cccccccc|}
    \hline
   & $x_1$ & $x_2$ & $x_3$   & $x_4$ & $x_5$   & $x_6$   & $x_7$   \\
$h_0=$& $0$   & $0$   & $\star$ & $0$   & $\star$ & $\star$ & $\star$ \\
$h_1=$& $0$   & $0$   & $\star$ & $1$   & $\star$ & $\star$ & $\star$ \\
$h_2=$& $0$   & $1$   & $\star$ & $\star$   & $0$ & $\star$ & $\star$ \\
$h_3=$& $0$   & $1$   & $\star$ & $\star$   & $1$ & $\star$ & $\star$ \\
$h_4=$& $1$   & $\star$   & $0$ & $\star$   & $\star$ & $0$ & $\star$ \\
$h_5=$& $1$   & $\star$   & $0$ & $\star$   & $\star$ & $1$ & $\star$ \\
$h_6=$& $1$   & $\star$   & $1$ & $\star$   & $\star$ & $\star$ & $0$ \\
$h_7=$& $1$   & $\star$   & $1$ & $\star$   & $\star$ & $\star$ & $1$ \\
\hline
\end{tabular}
\end{adjustbox}
\normalsize
\end{subfigure}
\caption{Consider $\cX$ as vertices of an infinite complete binary decision tree such that every internal vertex is labeled with a unique point $x\in\cX$, and define a partial concept class~${\cH_n\subset\{0,1,\star\}^\cX}$ which consists of all the partial concepts that realize exactly one branch to depth $n$, and label every point $x\in\cX$ outside of those $n$ vertices with $\star$. The Littlestone dimension of $\cH_n$ is~$n$, while the threshold dimension is $\leq 2$. Define $\cH=\bigcup \cH_n$, and it holds that $\LD{\cH}=\infty$, while $\TD{\cH}\leq 2$. This example can be easily modified to a multiclass $\cH$ over an infinite label domain $\cY$: for each hypothesis $h$, instead of labeling off-branch examples with $\star$, label them by a label $y_h\in\cY$ unique to $h$.}
\label{fig:class_with_td=2_ld=infty}
\end{figure}
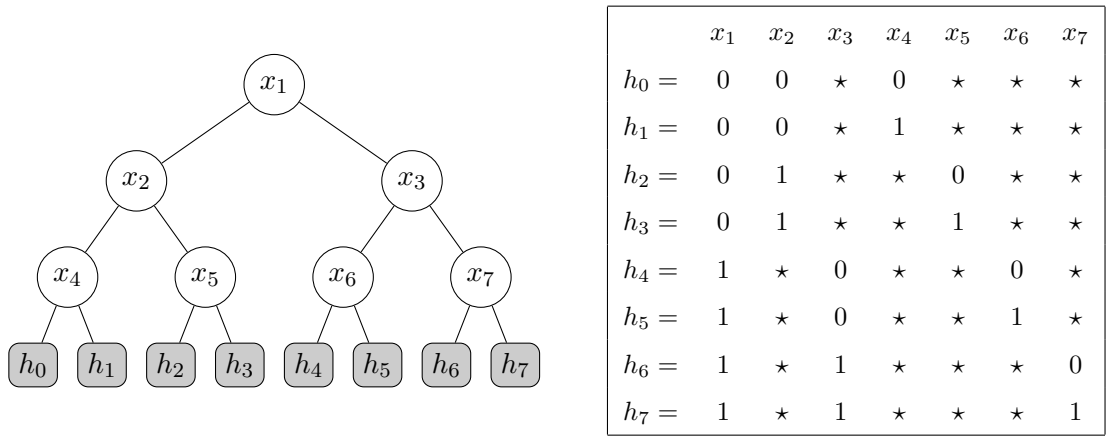

\paragraph{Step 1: Reduction to Comparison-based Predictions.}
The reduction to algorithms that learn thresholds
allows \citet{AlonBLMM22} to exploit the structure of one dimensional thresholds and use (the classic) Ramsey theorem to show the existence of what they refer to as an \emph{homogeneous set} of large size. Basically, a subset of an ordered domain $\cX$ is homogeneous with respect to an algorithm~$\cA$, if the prediction of $\cA$ on a test point $x$ depends only on the labels of the sorted input sample $S$ and the position of $x$ inside of $S$. 
This property amounts to the algorithm being \emph{comparison-based}, i.e.\ the algorithm makes all its decisions only based on how the elements of the input sample and the test point compare to each other, and not on their absolute values/locations.
We adapt this notion to trees in a natural way, and define \emph{comparison-based predictions} with respect to a binary decision tree. First, consider deterministic algorithms. Given a binary decision tree $T$, a deterministic algorithm is comparison-based with respect to $T$ if its prediction on a test point $x$ depends solely on the labels of the ordered input sample $S$, and the comparisons of the test point with points in $S$, where comparisons are based on the partial order on $T$. A result of comparing two examples $x'$ and $x''$ can be one of five outcomes: (i)+(ii) $x'$ is a left (right) descendant of $x''$, (iii)+(iv) $x''$ is a left (right) descendant of $x'$, or (v) $x'$ and $x''$ are incomparable.
For a visual example please refer to \Cref{fig:comparison_based}.
We extend this notion naturally to randomized algorithms; for more details and formal definitions please refer to~\Cref{sec:privacy_vs_LD}.

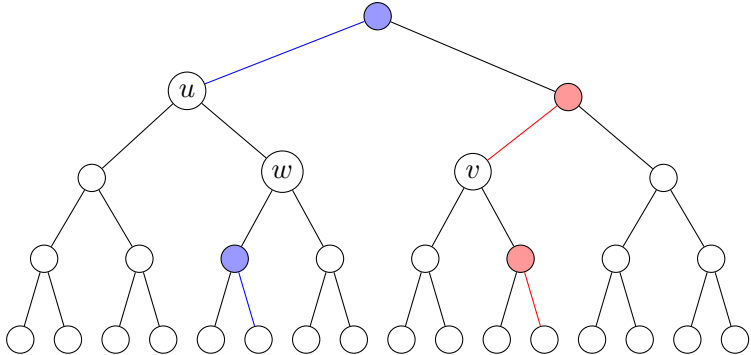
\begin{figure}[htb]
\centering
\begin{adjustbox}{valign=t}
\begin{forest}
rounded/.style={circle, draw, fill opacity=0.4}
    [{},fill=blue, for tree=rounded, edge=blue
      [{$u$},edge=blue,inner sep=0.8mm
        [{},
            [{},
                [{},]
                [{},]
            ]
            [{},
                [{},]
                [{},]
            ]
        ]
        [{$w$},inner sep=0.8mm
            [{},fill=blue
                [{},]
                [{},edge=blue]
            ]
            [{},
                [{},]
                [{},]
            ]
        ]
      ]
      [{},fill=red
        [{$v$},edge=red, ,inner sep=0.8mm
            [{},
                [{},]
                [{},]
            ]
            [{},fill=red
                [{},]
                [{},edge=red]
            ]
        ]
        [{},
            [{},
                [{},]
                [{},]
            ]
            [{},
                [{},]
                [{},]
            ]
        ]
      ]
    ]
\end{forest} \hspace{0.2cm}
\end{adjustbox}
\caption{The blue and the red labeled examples represent $T$-realizable (ordered) input samples  $S_{\textcolor{blue}{blue}}$ and $S_{\textcolor{red}{red}}$. The labels of $S_{\textcolor{blue}{blue}}$ and $S_{\textcolor{red}{red}}$ are the same, and the comparisons of the point $u$ and $S_{\textcolor{blue}{blue}}$ are identical to the comparisons of $v$ and~$S_{\textcolor{red}{red}}$. Therefore a comparison-based deterministic algorithm satisfies $\cA(S_{\textcolor{blue}{blue}})(u)=\cA(S_{\textcolor{red}{red}})(v)$. Note that the comparisons of $u$ and $S_{\textcolor{blue}{blue}}$ are not the same as the comparisons of $w$ and $S_{\textcolor{blue}{blue}}$, since the second point of $S_{\textcolor{blue}{blue}}$ is a right descendant of $u$ while it is a left descendant of $w$.
}
\label{fig:comparison_based}
\end{figure}

We developed our Ramsey result for chains in order to show that given a tree $T$, any algorithm is approximately comparison based on a large subtree. See \Cref{lemma:every_alg_is_comparison_based_somewhere} for more details.

\paragraph{Step 2: Lower Bound for Sample Complexity of Private Comparison-based Algorithms.}\label{par:bla} The second step of the proof consists of establishing a lower bound on the sample complexity of private comparison-based algorithms.
We do so by giving a reduction from the \emph{interior point problem}, which was introduced by \citet*{BunNSV15} in the context of properly learning thresholds. A randomized algorithm solves the interior point problem on~$[n]$ if for every input dataset $X\in[n]^m$, with high probability it returns a point that lies between $\min X$ and $\max X$ (see \Cref{sec:preliminaries_learning}). \citet{BunNSV15} showed that solving the interior point problem in a private manner requires a dataset size of $m\geq\Omega(\log^\star n)$. We use this result to derive a lower bound on the sample complexity. See \Cref{lemma:SC_of_CB_alg} for more details.

The idea that stands behind our reduction is that the predictions of empirical
comparison-based algorithms have high correlation with the branch on which the input sample $S$ lies on. To give some intuition about this idea, we will consider the following simple case. Let $T$ be a binary decision tree of depth $n$, and let $\cA$ be a deterministic algorithm that operates as follows. Given a $T$-realizable input sample $S$, $\cA$ finds the left most branch $B$ that realizes $S$ and outputs an hypothesis that labels every point on the branch according to the branch, and any point outside of it by~$0$. It might be convenient to imagine that $\cA$ outputs the branch $B$. Note that indeed~$\cA$ is comparison-based.
Now, given $m$ points $d_1,\ldots,d_m\in\left[\frac{n}{2}\right]$, representing an input dataset to the interior point problem, pick uniformly at random a branch $B$ in $T$, and associate each point $d_i$ with the point $x_i$ on $B$ of depth $d_i$. The first $m/2$ points define an input sequence $S=\left((x_1,y_1),\ldots, (x_{m/2},y_{m/2})\right)$, where the labels $y_i$'s are determined by~$B$. Next, run~$\cA$ on $S$ and denote its output by~$B_\cA$. Since~$B$ was chosen randomly, and $\cA$ can only see up to first  $\frac{n}{2}$ turns of~$B$, the depth of the first point on which $B$ and $B_\cA$ deviates is (w.h.p.) an interior point of $d_1,\ldots,d_m$.  

We note that in contrast with the first part of our proof (reduction to comparison based algorithms), which can be seen as an adaptation of the first part of the proof by \citet{AlonBLMM22} (with Ramsey for trees replacing the classical Ramsey applied on thresholds), the second part in our proof (involving a reduction from the interior point problem) is not an adaptation of the proof for the binary case. Instead, \citet{AlonBLMM22} provided a direct privacy attack on homogeneous algorithms that learn thresholds. Their attack is constructive and relies on specific properties of homogeneous algorithms and the structure of one-dimensional thresholds. It remains unclear whether this attack can be adapted to comparison-based predictions within the context of trees. Despite our best efforts to construct a similar attack, we were unable to do so successfully.

\medskip

Quantitatively, the two steps in our proof combine a bound on Ramsey numbers (Step 1) and a lower bound for the interior point problem (Step 2). At first sight, it might seem that composing these yields an $\Omega(\log^\star\log^\star d)$ lower bound rather than the single $\Omega(\log^\star  d)$ which is stated in \Cref{thm:DP_implies_LD}. This initial impression arises because each of these arguments independently involves a $\log^\star$ cost. The Ramsey bound in the first step is of the same magnitude as the one in the proof by \citet{AlonBLMM22}, which yields a $\log^\star$ cost. Additionally, the interior point problem also has a $\log^\star$ cost~\cite{BunNSV15}. Hence, composing them sequentially might intuitively suggest that the resulting bound would be $\Omega(\log^\star\log^\star(d))$. However, a more careful handling of the implied calculations yields a single $\log^\star$, which is optimal~\cite{KaplanLMNS20}.

\medskip
Finally, it is worth noting that our proof not only generalizes the proof for the binary case to broader settings, but also, by directly addressing Littlestone trees instead of taking an indirect path through thresholds, illustrates that the Littlestone dimension serves as an inherent barrier to privacy. This highlights a robust connection between the two, showcasing them as two forms of~stability.

\subsection{Ramsey Theorems} \label{sec:technical:Ramsey}
In this section, we outline the main steps and ideas in the proofs of our Ramsey results on binary trees.
The counterexamples shown in~\Cref{sec:main_results:ramsey} indicate that, if we color subsets of vertices without differentiating according to a notion of type, no positive existence result is attainable.
The approach we take is to consider colorings of subsets of vertices of a \emph{specific type}, rather than colorings of all possible subsets of $m$-vertices. We then demonstrate that for any such coloring, there exists a deep monochromatic subtree. Then, given an arbitrary coloring of all $m$-subsets, we apply our argument to each of the $\tau(m)$ different types and obtain a subtree that contains at most $\tau(m)$ ($=$ number of types) many colors.\footnote{This approach yields loose quantitative upper bounds on the Ramsey numbers. In~\Cref{thm:ramsey_chains}, we obtain better bounds by deriving the type-monochromatic subtree while considering all types 'in parallel', rather than applying the type-specific bound sequentially 'type-by-type'.} The obtained subtree is in fact type-monochromatic:  whenever two $m$-subsets are of the same type, they are colored with the same color.

The distinction between incomparable and comparable pairs plays a key role in our results. The quantitative bounds we obtain in these two cases are very different, and in fact, in the general case $m\geq 2$, we provide quantitative results only for $m$-chains: i.e.\ subsets of vertices that do not contain incomparable pairs.
This distinction is reflected in the infinite case as well.
Indeed, we claim that for any finite coloring of all subsets from a given type $\tau$ of an infinite binary tree, there exists an infinite monochromatic subtree if and only if $\tau$ is a chain. If $\tau$ is not a chain then one can only guarantee monochromatic subtrees of (arbitrarily large) finite depth (see \cref{app:ramsey_infinite}).
In the proof overview we focus on the case of pairs ($m=2$), and, as noted above, consider colorings of pairs of a specific type. This case is not only simpler to describe but also effectively captures the central ideas present in the proofs of the general results for $m$-subsets and $m$-chains. 

\paragraph{Comparable Pairs ($\equiv$ Chains).}
The proof for comparable pairs (as well as for chains) follows the same inductive procedure as described in the proof by \citet{ErdosRado52} for the classic Ramsey theorem, with this procedure being applied in parallel across the branches.
The building block of the proof is the pigeonhole principle for trees: for every vertex-coloring of a complete binary tree of depth~$n$ with~$k$ colors, there exists a monochromatic subtree of depth $\left\lfloor\frac{n}{k}\right\rfloor$ (see \Cref{prop:php}).

Assume without loss of generality that the given coloring is a coloring of \emph{left}-pairs\footnote{A pair of vertices is a \emph{left}-pair if one of the vertices is a left descendant of the other.}. We inductively define an appropriate vertex-coloring such that the guaranteed monochromatic subtree from the pigeonhole principle is also monochromatic with respect to the original coloring. 
The desired vertex-coloring is defined as follows. We start from the root of~$T$, and color every vertex in its left subtree by the color of the pair consisting of itself and the root. By applying the pigeonhole principle, there exists a monochromatic subtree~$S$ of the left subtree. The color of the root is set to be exactly the color of $S$. We repeat the exact same process and color in the same manner the roots of~$S$ and of the right subtree of~$T$, and inductively proceed in this way. Observe that the subset of vertices that we color via this procedure is a subtree of $T$. Moreover, the following property holds: if $v$ is colored with the color~$c$, then for every colored vertex $u$ that is a left descendant of $v$, the pair $\{v,u\}$ is also colored in the color~$c$ in the original coloring. Therefore, a monochromatic subtree with respect to this vertex-coloring is indeed also monochromatic with respect to the original \emph{left}-pairs coloring.

It is interesting to note that the existence of Ramsey numbers for $m$-chains can be derived as a corollary from Milliken's theorem~\cite{Milliken79}.
The finite version of Milliken's theorem states that for every positive integers $d,m$ and $k$, there exists a number $n=n(d,m,k)$ such that the following holds: for every complete binary tree $T$ of depth~$n$ and every coloring of level-aligned subtrees of~$T$ of depth~$m$ with~$k$ colors, there exists a monochromatic level-aligned subtree of~$T$ of depth~$d$.
Indeed, a coloring of $m$-chains of a fixed type $\tau$ induces a coloring of level-aligned subtrees of depth $m-1$, since every level-aligned subtree contains exactly one $m$-chain of type $\tau$. Therefore, a weaker, qualitative version of our result for $m$-chains follows by Milliken.
However, as we mentioned above, Milliken's proof does not provide quantitative bounds and the methodology used to prove it does not make it easy to obtain them as well. Our proof, instead, follows the line of the proof of the quantitative version of Ramsey theorem by~\citet{ErdosRado52} and thus allows to obtain non-trivial upper bounds for the Ramsey numbers.

\paragraph{Incomparable Pairs ($\equiv$ Non-Chains).}

The case of incomparable pairs (as well as the one of general non-chain $m$-subsets) is more intricate, as indicated by the great difference between the lower bound obtained in this case and the one for comparable pairs.
Our proof is based on an intermediate bipartite Ramsey result, which serves as a key component of it and might be of independent interest.
Informally, the bipartite Ramsey theorem says the following. Let $L$ and $R$ be complete binary trees. Then, for every coloring of pairs of vertices $\{\ell,r\}$, where $\ell\in L$ and $r\in R$ there exist deep\footnote{The depth of the obtained subtrees is a function of the depths of $L$ and $R$ and the number of colors.} subtrees $L'$ and $R'$
of $L$ and $R$, such that all pairs $\{\ell',r'\}, \ell'\in L', r'\in R'$ have the same color. For a formal statement see \Cref{theorem:bipartite}.

We begin by outlining the proof for incomparable pairs under the assumption of the bipartite theorem, followed by an explanation of how we establish it.
Given a coloring of incomparable pairs of a tree $T$, we define a vertex-coloring as follows. First consider the root of $T$, and its left and right subtrees. 
By applying the bipartite theorem on the left and right subtrees we obtain subtrees~$L'$ and~$R'$ such that all incomparable pairs of vertices from $L'$ and $R'$ are colored with the same color. Set the color of the root to be exactly this color. 

Repeat the same process and color the roots of~$L'$ and $R'$ in the same manner. This procedure yields a vertex-coloring of a subtree~$S$ of~$T$, with the property that the color of every incomparable pair $\{u,v\}$ in~$S$ with respect to the original pair-coloring is identical to the color of $\LCA(u,v)$ in~$S$. Finally, by applying the pigeonhole principle on $S$, we obtain a subtree that is monochromatic with respect to original pair-coloring, as desired. 

We proceed by outlining the proof of the bipartite theorem.
Let $L$ and $R$ be (deep enough) complete binary trees and consider a coloring of pairs $\{l,r\}$, where $l\in L, r\in R$. We start these interleaved colorings of vertices in $L$ and vertices in $R$:
\begin{enumerate}
    \item Given a vertex $v\in R$, color every vertex $u\in L$ by the color of the pair $\{u,v\}$. Apply the pigeonhole principle on $L$ and obtain a monochromatic subtree $T_v$ of depth $d$ colored with $c_v$. 
    \item Color $v\in R$ by $(T_v,c_v)$. 
    \item Do the same for all vertices in $R$. Apply the pigeonhole principle on $R$ and obtain a monochromatic subtree $R'$ colored with $(L',c)$, where $L'$ is a subtree of $L$.

    where its color is a subtree of $L$, which we call $L'$, together with a color~$c$.
\end{enumerate}
By construction, all pairs $\{l',r'\}$ where $l'\in L', r'\in R'$ are colored with the same color which is exactly~$c$.
Note that the number of colors used in the vertex coloring of $R$ is 
\[\lvert\{\text{subtrees of $L$ of depth $d$}\}\rvert\cdot (\text{number of colors in original coloring}).\]
Therefore, in order for $R'$ to be deep, we need to control the number of colors, and hence $L$ needs to be much shallower than $R$. On the other hand, in order for $L'$ to be deep, $L$ must be sufficiently deep.

\section{Preliminaries}\label{sec:preliminaries}

\subsection{Combinatorics}\label{sec:preliminaries_combi}

\begin{definition}[Subtree]\label{def:subtree}
Let $T$ be a binary tree. Define a \emph{subtree} of $T$ by induction on its depth~$d$. All vertices of $T$ are subtrees of $T$ of depth $d=0$. For $d \ge 1$ a subtree
of depth $d$ is obtained from an internal vertex of $T$ and a subtree of depth $d-1$ of the tree rooted at its left child, and a subtree of depth $d-1$ of the tree rooted at its right child. A subtree $S$ of $T$ is \emph{level-aligned} if any two vertices in the same level in $S$ are also in the same level in $T$.
\end{definition}

\begin{proposition}[Pigeonhole Principle for Trees, \cite{AlonBLMM22,PachTS12}]\label{prop:php}
  Let $d\in\bbN$ and let $T$ be a complete binary tree of depth $n$. Then, for every coloring of its vertices with $k$ colors, the following hold:
  \begin{enumerate}[label=(\roman*)]
  \item If $T$ has depth $n \ge dk$, it admits a subtree $S$ of depth $d$;\label{item:php}
   \item If $T$ has depth $n \ge 5dk\log k$, it admits a level-aligned subtree $S$ of depth $d$.
  \end{enumerate}
\end{proposition}

We remark that \citet{AlonBLMM22} gave a proof for \Cref{item:php} only for $k=2$. 
For completeness, we provide a full proof for \Cref{item:php} in \Cref{app:additional_proofs:php}.

\subsection{Learning}\label{sec:preliminaries_learning}

\paragraph{PAC Learning.}
We use standard notations from PAC learning; for more details see e.g.\ \cite{shalev2014understanding}.
Let $\cX$ be a domain and~$\cY$ a label space.
A partial concept is a partial function $h$ from~$\cX$ to~$\cY$. It is convenient to represent a partial concept as a total function $h\in(\cY\cup \{\star\})^\cX$, where $\star\notin \cY$, and $h(x)=\star$ means that $h$ is undefined on $x$. The support of a partial concept $h$ is ${\mathsf{supp}(h)=\{x\in\cX|h(x)\neq \star\}}$.
Given an hypothesis $h:\cX\to \cY$, the \emph{empirical loss} of $h$ with respect to a sample $S=\bigl((x_1,y_1),\ldots,(x_m,y_m)\bigr)\in (\cX \times \cY)^m$ is defined as $\loss{S}{h}\coloneqq\frac{1}{m}\sum_{i=1}^m\1[h(x_i)\neq y_i]$.
A sample $S=\bigl((x_1,y_1),\ldots,(x_m,y_m)\bigr)\in (\cX \times \cY)^m$ is \emph{realizable} by a (possibly partial) concept class~$\cH$ if there is $h\in\cH$ such that $\{x_1,\ldots,x_m\}\subset \mathsf{supp}(h)$ and $\loss{S}{h}=0$.
The \emph{population loss} of $h$ with respect to a distribution $\cD$ over $\cX\times\cY$ is defined as $\loss{\cD}{h}\coloneqq\Pr_{(x,y)\sim \cD}[h(x)\neq y]$.
A distribution $\cD$ over labeled examples is \emph{realizable} with respect to~$\cH$ if $\inf_{h\in\cH}\loss{\cD}{h}=0$.

For a set $Z$, let $Z^\star= \cup_{n=0}^\infty Z^n$.
A \emph{learning rule} $\cA$ is a (possibly randomized) algorithm that takes as input a sample $S \in (\cX \times \cY)^\star$ and outputs an hypothesis~${h=\cA(S)\in \cY^\cX}$. 
In the realizable PAC learning model, the input $S$ is sampled i.i.d. from a realizable distribution~$\cD$, and the learner's goal is to output an hypothesis with small population loss with respect to $\cD$. 
More precisely, let $m, \alpha, \beta >0$. 
We say that an algorithm $\cA$ is an \emph{$(\alpha, \beta)$-accurate PAC learner} for $\cH$, with sample complexity $m$, if for every realizable distribution $\cD$, $\Pr_{S \sim \cD^m}\left[ \loss{\cD}{\cA(S)}\geq\alpha\right]\leq\beta$. Here, $\alpha$ is called the \emph{error} and~$\beta$ is called the \emph{confidence parameter}.
A class $\cH$ is \emph{PAC learnable} if there exist vanishing $\alpha(m),\beta(m)\to 0$ and an algorithm $\cA$ such that for all $m$, $\cA$ is a $(\alpha(m),\beta(m))$-learner for~$\cH$ with sample complexity $m$.

\paragraph{Littlestone Dimension.}

The Littlestone dimension is a combinatorial parameter which captures mistake and regret bounds in online learning \cite{Lit88,Ben-DavidPS09}. The definition of the Littlestone dimension uses the notion of \emph{mistake trees}. 
A mistake tree is a binary decision tree whose internal vertices are labeled with instances from~$\cX$, and whose edges are labeled by instances from a label space~$\cY$ such that each internal vertex has different labels on its two outgoing edges. A root-to-leaf path in a mistake tree is a sequence of labeled examples $(x_1,y_1),\dots,(x_d,y_d)$. 
The point $x_i$ is the label of the $i$'th internal vertex in the path, and $y_i$ is the label of its outgoing edge to the next vertex in the path.
We say that a class~$\cH$ \emph{shatters} a mistake tree if every root-to-leaf path is realizable by $\cH$.
The \emph{Littlestone dimension} of~$\cH$, denoted $\LD{\cH}$, is the largest number $d$ such that there exists a complete mistake tree of depth $d$ shattered by~$\cH$.
If $\cH$ shatters arbitrarily deep mistake trees then we write $\LD{\cH}=\infty$.

\paragraph{Threshold Dimension.} The Threshold dimension $\TD{\cH}$ of a binary concept class $\cH \subset \{0,1\}^\cX$ is the maximal number $d$ such that there exist $x_1,\ldots,x_d\in \cX$ and $f_1,\ldots, f_d \in \cH$ such that 
$f_i(x_j)=\1[i\geq j]$ for all $i,j\in [d]$.
If there are such $x_1,\ldots,x_d,f_1,\ldots, f_d$ for every arbitrarily large~$d$ then we write $\TD{\cH}=\infty$.

\paragraph{Differential Privacy.}
We use standard notations from differential privacy literature; for more details see e.g. \cite{DR14, Vadhan17survey}.
    Let $\epsilon,\delta \geq 0$. 
    For two numbers $a$ and $b$, denote $a\overset{\epsilon,\delta}{\approx}b$ if $a\leq e^\epsilon\cdot b +\delta$, and $b\leq e^\epsilon\cdot a +\delta$.

\begin{definition}[Differential Privacy]\label{definition:dp-stable}
    Let $\epsilon,\delta\geq 0$. A learning rule $\cA$ is $(\epsilon, \delta)$-differentially private if for
    every pair of training samples $S, S' \in (\cX\times\cY)^m$ 
    that differ on a single example, and every event~$E\subseteq \cY^\cX$, 
    \[\Pr[\cA(S)\in E]\overset{\epsilon,\delta}{\approx}\Pr[\cA(S')\in E].\]
\end{definition}
Typically, $\epsilon$ is chosen to be a small constant (e.g. $\epsilon \leq 0.1$) and $\delta$ is negligible (i.e.\ ${\delta(m) \leq m^{-\omega(1)}}$).
An hypothesis class $\cH$ is \emph{privately learnable} (abbreviated \emph{DP learnable}) if it is PAC learnable by an algorithm $\cA$ which is $(\epsilon(m),\delta(m))$-differentially private, where $\epsilon(m)=O(1)$ is a numerical constant and $\delta(m)=m^{-\omega(1)}$.

A basic property of differential privacy is that privacy is preserved under post-processing; it enables arbitrary data-independent transformations to differentially private outputs without affecting their privacy guarantees. 

\begin{proposition}[Post-Processing]\label{prop:dp_post_processing}
   Let $\cM:\cZ^m\to\cR$ be any $(\epsilon,\delta)$-differentially private algorithm, and let $f:\cR\to\cR'$ be an arbitrary randomized mapping. Then $f\circ\cM:\cZ^m\to\cR'$ is $(\epsilon,\delta)$-differentially private.
\end{proposition}

\paragraph{Empirical Learners.}
    Let $\cH$ be a class. An algorithm $\cA$ is $(\alpha,\beta)$-accurate empirical learner with sample complexity $m$ if for every realizable sample $S$ of size $m$, 
    \[\Pr_{h\sim\cA(S)}[\loss{S}{h}\geq \alpha]\leq \beta.\]

\citet*{BunNSV15} proved that any private PAC learner can be transformed into a private empirical learner, while the sample complexity is increased only by a multiplicative constant factor. 

\begin{lemma}[Lemma 5.9 in \cite{BunNSV15}]\label{lemma:reducrion_to_empirical_learner}
    Suppose $\cA$ is an $(\epsilon, \delta)$-differentially private $(\alpha, \beta)$-accurate PAC learner for an hypothesis class $\cH$ with sample complexity $m$. Then there is an $(\epsilon, \delta)$-differentially private $(\alpha, \beta)$-accurate empirical learner for $\cH$ with sample complexity $n = 9m$.
\end{lemma}
We remark that \cref{lemma:reducrion_to_empirical_learner} was proved in \cite{BunNSV15} within the context of binary (total) classes. However, the proof also applies to partial concept classes with arbitrary label spaces.

\paragraph{Interior Point Problem.}
An algorithm $\cA : [n]^m \to [n]$ solves the interior point problem on $[n]$ with probability $1-\beta$ if for every input sequence $x_1\ldots x_m\in [n]$,
\[\Pr[\min x_i\leq \cA(x_1,\ldots x_m)\leq \max x_i]\geq 1-\beta\]
where the probability is taken over the randomness of $\cA$; the number of data points $m$ is called the sample complexity of $\cA$.

\begin{theorem}[Theorem 3.2 in \cite{BunNSV15}]
Let $0<\epsilon< 1/4$ be a fix number and let $\delta(m)\leq 1/50m^2$. Then for every positive integer~$m$, solving the interior point problem on $[n]$ with probability at least~$3/4$ and with $(\epsilon,\delta(m))$-differential privacy requires sample complexity $m\geq \Omega(\log^\star n)$.
\end{theorem}

\begin{lemma}\label{lemma:rescaling_ipp}
    Let $\cA:[n]^m\to[n]$ be an $(\epsilon,\delta)$-differentially private algorithm, and let $C(n)\leq \log ^2 n$. Assume that for every input sequence $x_1,\ldots,x_m$ such that $min_{i\neq j}|x_i-x_j|\geq C(n)$,
    \[\Pr[\min x_i\leq \cA(x_1,\ldots,x_m)\leq \max x_i]\geq\frac{3}{4}.\]
    Then, $m\geq\Omega(\log^\star n)$.
\end{lemma}

\begin{proof}
    By rescaling, $\cA$ solves the interior point problem on a domain of size~$\frac{n}{C(n)}$, and therefore the sample complexity satisfies $m\geq \Omega\left(\log^\star\left(\frac{n}{C(n)}\right)\right)= \Omega\left(\log^\star\left( \frac{n}{\log^2 n}\right)\right)= \Omega(\log^\star n)$.
    Indeed, the second equality holds since by the definition of the $\log^\star$ function, $\twr_{(t+2)}(1)\geq2^{\frac{n}{\log^2 n}}$, where $t=\log^\star\left(\frac{n}{\log^2 n}\right)$.
    Furthermore, $2^{\frac{n}{\log^2 n}}\geq n$ for large enough $n$.
    Therefore, again by definition, $t+1\geq \log^\star n$ for large enough $n$, which implies that  $\log^\star\left(\frac{n}{\log^2 n}\right)\geq \frac{1}{2}\log^\star n$ for large enough $n$.
\end{proof}

\subsection{Additional Notations}
\label{sec:additional_defs}

\paragraph{Iterated Logarithm and Tower Function.}
    The tower function $\twr_{(t)}(x)$ is defined by the recursion
    \begin{equation*}
        \twr_{(i)}(x)=\begin{cases}
            x & i=1;\\
            2^{\twr_{(i-1)}(x)} & i>1.
        \end{cases}
    \end{equation*}
    The iterated logarithm $\log _{(t)}(x)$ is defined by the recursion
    \begin{equation*}
        \log_{(i)}(x)=\begin{cases}
            x & i=0;\\
            \log({\log_{(i-1)}(x)}) & i>0.
        \end{cases}
    \end{equation*}
    Note that for all $t$, \big($\twr_{(t)}(\cdot)\big)^{-1}=\log_{(t-1)}(\cdot)$. Finally,
    \[\log^\star(x)=\min\{t\mid \log_{(t)}(x)\leq 1\}.\]

\section{A Ramsey Theory for Trees}\label{sec:ramsey_theory_in_trees}

This section contains all the proofs of the Ramsey theorems for Trees by the order of appearance in \Cref{sec:main_results:ramsey}.

\subsection{Pairs: Proof of  \Cref{thm:ramsey_pairs} (Warm up)}\label{sec:ramsey_pairs_proof}
In this section, we prove \Cref{thm:ramsey_pairs}.
In order to prove the existence of a trichromatic subtree, we take the following approach. Instead of considering colorings of all possible pairs, we consider colorings of all pairs of a specific type and prove that, for any such coloring, there exists a \textbf{mono}chromatic subtree. Therefore, given an arbitrary coloring of all pairs, by applying our arguments sequentially we are able to obtain a subtree that is trichromatic and also has the property that each pair is colored according to its type.

Let us call a pair of vertices $\{u,v\}$ an $\mathsf{L}$-pair ($\mathsf{R}$-pair) if one of $u$ or $v$ is a left (right) descendant of the other and a $\mathsf{U}$-pair if $u$ and $v$ are incomparable.

\begin{proposition}[Comparable Pairs]\label{prop:ramsey_pairs_comp}
    Let $d,k\in \bbN$ and $\tau \in \{\mathsf{L},\mathsf{R}\}$ be fixed. Then there exists an $n=n(d,\tau,k)\in\bbN$ such that for every binary tree $T$ of depth $n$ and every coloring of $\tau$-pairs in $T$ with $k$ colors, $T$ admits a monochromatic subtree of depth $d$. Moreover,
    \[n(d,\tau,k)\leq 2^{dk\log (2k)}.\]
\end{proposition}

\begin{proposition}[Incomparable Pairs]\label{prop:ramsey_pairs_incomp}
    Let $d,k\in \bbN$. Then there exists an $n=n(d,\mathsf{U},k)\in\bbN$ such that for every binary tree $T$ of depth $n$ and every coloring of $\mathsf{U}$-pairs in $T$ with $k$ colors, $T$ admits a monochromatic subtree of depth $d$. Moreover,
    \[n(d,\mathsf{U},k)\leq \twr_{(3dk+\log^\star(2^{2^{4k}})+1)}(1).\]
\end{proposition}

\Cref{thm:ramsey_pairs} is an immediate corollary of \Cref{prop:ramsey_pairs_comp,prop:ramsey_pairs_incomp}.
Indeed, let $d\in\bbN$ be the depth of the desired trichromatic subtree. Set $n=n(n(n(d,\mathsf{L},k),\mathsf{R},k),\mathsf{U},k)$, then for any coloring of all pairs of a tree of depth $n$, there exists a subtree of depth $d$ that is trichromatic and also type-monochromatic, in the sense that if two pairs are of the same type then they are colored with the same color.
Indeed, the outer-most $n(\cdot,\cdot,\cdot)$ guarantees a $\mathsf{U}$-monochromaric subtree of depth $n(n(d,\mathsf{L},k),\mathsf{R},k)$, then the second $n(\cdot,\cdot,\cdot)$ yields a subtree of depth $n(d,\mathsf{L},k$) that is type-monochromatic with respect to $\{\mathsf{R},\mathsf{U}\}$, and finally the inner-most $n(\cdot,\cdot,\cdot)$ guarantees a subtree of depth $d$ that is type-monochromatic with respect to all types $\{\mathsf{L}, \mathsf{R},\mathsf{U}\}$.

We turn to prove \Cref{prop:ramsey_pairs_comp,prop:ramsey_pairs_incomp}.

\subsubsection{Comparable Pairs: Proof of \Cref{prop:ramsey_pairs_comp}}
\begin{proof}
    Let $T$ be a complete binary tree of depth $n$ large enough, and let $\chi$ be a coloring of comparable pairs of a fixed type $\tau \in \{\mathsf{L},\mathsf{R}\}$ with $k$ colors. We can assume without loss of generality that~$\tau = \mathsf{L}$, i.e.\ we are coloring left pairs. We prove the result by constructing the desired subtree in the following way. First, we introduce a recursive procedure that constructs a subtree of~$T$, denoted as~$T^{\star}$ and define a suitable ``meta-coloring" of the vertices of~$T^{\star}$, which we call $\chi^{\star}$. Then, we apply~\Cref{prop:php} on $T^{\star}$, obtaining a monochromatic subtree $T'$ with respect to $\chi^{\star}$. Thanks to the definition of $\chi^{\star}$, the subtree $T'$ will have the desired property that we seek.
    
    Let $r$ be the root of $T$ and assign to each vertex $v$ in the left subtree of $r$ the color $\chi(\{r,v\})$. By~\Cref{prop:php} there exists a monochromatic subtree $S_0$ of the left subtree of color $c$.
    Let $\chi^{\star}(r) = c$ and call $S_1$ the right subtree of $r$.
    Now, for the next step, we can repeat this procedure on $S_0, S_1$, color their roots, and obtain two new subtrees for each of $S_0, S_1$.
    In general, for each subtree $S$ in the current step, we apply the same argument, color its root, and get the two subtrees~$S_0,S_1$ of $S$ for the next step.
    Let $S$ be any subtree received by this recursive procedure at some point, and let $S_0,S_1$ be the two subtrees received after applying the reduction step on $S$, then the following relations hold:
    \begin{align*}
    \mathtt{depth}(S_0)&\geq \left\lfloor\frac{\mathtt{depth}(S)-1}{k}\right\rfloor,\\ 
    \mathtt{depth}(S_1)&= \mathtt{depth}(S) -1.
    \end{align*}

 Let $t$ be the maximal number of steps this procedure could be repeated, such that all subtrees so far were non-empty.
    Let $T^{\star}$ be the subtree formed by 
    the colored roots of all subtrees $S$ in this process.
    By one last application of the pigeonhole principle on $T^{\star}$ and $\chi^\star$ we obtain a monochromatic subtree $T'$ of depth $d = \lfloor t/k\rfloor$. We argue that $T'$ is monochromatic with respect to the original $\chi$. Indeed, if $\{v,w\}$ is any left pair in $T'$, $\chi(\{v,w\}) = \chi^{\star}(v)$ by our construction and by the definition of $\chi^{\star}$, which implies the desired result.

 It is left to prove a bound over $t = dk$. 
 For convenience, for every time step $i\in\{0,1,\ldots, t\}$ denote the depth of the subtree $S$ by $d_i$. Therefore 
 \begin{equation*}
      \left\{
    	\begin{array}{ll}
    d_{0} &= n,\\
    d_{i+1} &\ge \left\lfloor\frac{d_i -1}{k}\right\rfloor.
    	\end{array}
    \right.
 \end{equation*}
 Note that $\left\lfloor\frac{d_i-1}{k}\right\rfloor\geq \frac{d_i}{2k}$ for every step $i<t$ (otherwise, $\left\lfloor\frac{d_i-1}{k}\right\rfloor=0$ and the procedure must terminate after this step). 
 For the procedure to continue until time $t$ we must have $d_t\geq 0$, and this implies that $\frac{n}{(2k)^t}\geq 1$.  As a consequence,  $n \ge (2k)^{dk} = 2^{dk \log (2k)}$ is sufficient to guarantee that the depth of $S$ is $d$, giving the desired bound on the associated Ramsey number.
\end{proof}

\subsubsection{Incomparable pairs: Proof of \Cref{prop:ramsey_pairs_incomp}}

The technique used in this case is quite different from the one used before, but the main idea stays the same: i.e.\ we will define a suitable coloring of vertices of $T$ via a repeated use of~\Cref{prop:php}. 
The following theorem is the key element in the proof of the result for incomparable pairs.

\begin{theorem}[Bipartite Theorem]\label{theorem:bipartite}
     Let $L$ and $R$ be complete binary trees of depth $n$. Then, for every coloring with~$k$ colors of all pairs of the form $\{\ell,r\}$, where $\ell\in L$ and $r\in R$, there exist subtrees $L'$ of $L$ and $R'$ of $R$ of depth \[d(n,k)\geq\frac{\log_{(2)}(5n\log k)}{15k\log k}\]  such that all pairs $\{\ell',r'\}$, where $\ell'\in L'$ and $r'\in R'$, are colored with the same color.
     Moreover, if $n\geq 2^{2^{4k}}$, then $d(n,k)\geq \log_{(3)}(n+1)$.
\end{theorem}

We begin with proving \cref{prop:ramsey_pairs_incomp}, the result for incomparable pairs, assuming \cref{theorem:bipartite}. 
First, observe that any incomparable pair $\{u,v\}$ has a lowest common ancestor $\LCA(u,v)$ that is distinct from both of them such that one of $u,v$ is a left descendant of $\LCA(u,v)$ and the other is a right descendant of $\LCA(u,v)$.
In the course of the proof, we will define a new coloring $\chi^{\star}$ with $k$ colors over internal vertices $w$, based on the coloring of pairs $\{u,v\}$ such that $\LCA(u,v) = w$. The new coloring $\chi^{\star}$ is defined in a recursive procedure where,  as mentioned above, the key component we use is \cref{theorem:bipartite}.

\begin{proof}[Proof of \Cref{prop:ramsey_pairs_incomp}]

Let $\chi$ be a coloring of incomparable pairs with $k$ colors.
In the course of the proof we define recursively a subtree $T^\star$ of $T$ of depth $dk$, and define a vertex coloring $\chi^\star$ of $T^\star$ with the following property: for every  incomparable pair $\{u,v\}$ in $T^\star$,
\[\chi(\{u,v\})=\chi^\star(\LCA(u,v)),\]
where $\LCA(u,v)$ is the lowest-common-ancestor of $u$ and $v$ in $T^\star$. By applying the pigeonhole principle on $T^\star$, we obtain a $\chi^\star$-monochromatic subtree of depth $d$,
and by the construction and the above property, this subtree is also monochromatic with respect to the original pairs-coloring~$\chi$, as desired.

The main building block used in the inductive procedure is \Cref{theorem:bipartite}. Throughout the proof we denote the vertices of $T^\star$ by $u_\emptyset,u_0,u_1,u_{00},u_{01},\ldots,u_{1^{dk}}$,where for a string $\sigma\in\{0,1\}^i$, $i\in\{0,\ldots, dk-1\}$, $u_{\sigma 0}$ is the left descendant of $u_\sigma$, and $u_{\sigma 1}$ is the right descendant of $u_\sigma$.

\medskip

Define $S_\emptyset=T$ and define $u_\emptyset$ to be the root of $S_\emptyset$. Denote $d_0=n$ the depth of $S_\emptyset$.
Denote by $T_0$ and $T_1$ the left and right subtrees of $S_\emptyset$, respectively.
By applying \Cref{theorem:bipartite}, we obtain subtrees~$S_0$ of $T_0$ and $S_1$ of $T_1$, of depth 
\[d_1=d(d_0-1,k),\]
such that all of the pairs $\{u,v\}$ with $u\in S_0$ and $v\in S_1$ are colored with the same color. We define $\chi^\star(u_\emptyset)$ to be exactly that color. 

We repeat this procedure inductively over $S_0$ and $S_1$ separately.
Assume that $S_\sigma$ is defined where $\sigma\in\{0,1\}^i$ for $i\geq 0$, and $u_\sigma$ is the root of $S_\sigma$.
Let $T_{\sigma0}$ and $T_{\sigma1}$ be the left and right subtrees of $S_\sigma$. By applying \Cref{theorem:bipartite} on $T_{\sigma0}$ and $T_{\sigma1}$, we obtain subtrees $S_{\sigma0}$, $S_{\sigma1}$ of depth 
\[d_{i+1}=d(d_i-1,k),\]
such that all incomparable pair $\{u,v\}$ with $u\in S_{\sigma0}$ and $v\in S_{\sigma1}$ are colored with the same color. Define $\chi^\star(u_\sigma)$ to be exactly that color.

\medskip

Assume that $n$ is sufficiently large so this procedure can be continued $t=dk$ steps, and for every step~$i\geq 1$, 
\[d_i=d(d_{i-1}-1,k)\geq \log_{(3)}(d_{i-1}),\] 
where the value of $n$ will be specified later on. 
By applying this relation $t$ times, we obtain 
\[d_t\geq \log_{(3t)}(n).\]
To ensure that the procedure can be continued $t$ steps, by \cref{theorem:bipartite} it suffices to require $\log_{(3t)}(n)\geq 2^{2^{4k}}=:n_0$. Equivalently $3t=3dk\leq \log^\star(n)-\log^\star(n_0)$.
To conclude, by reversing the relation, there exists a monochromatic subtree of depth $d$ if 
\[n\geq \twr_{(3dk+\log^\star(n_0) +1)}(1).\]

\end{proof}

We now turn to prove \cref{theorem:bipartite}. We begin with proving a slightly more general version of \cref{theorem:bipartite}, where the depths of $L$ and $R$ may differ.

\begin{lemma}[Bipartite Lemma]\label{lem:bipartite_antichains}
    Let $L$ and $R$ be trees of depth $d_L$ and $d_R$, respectively. Then, if $\chi$ is a coloring of all pairs of the form $\{\ell,r\},\ell\in L,r\in R$ with~$k$ colors, then there exist subtrees~$L'$ of~$L$ and~$R'$ of~$R$ such that (i) all pairs $\{\ell,r\},\ell\in L',r\in R'$ have the same color, (ii) the depth of $L'$ is at least $\frac{d_L}{5k\log k}$, and the depth of $R'$ is at least $\frac{d_R}{k\alpha(d_L,k)}$, where $\alpha(d,k) = \binom{d}{1+d/5k\log k} \cdot 2^{d2^{d/5k\log k}}$.
\end{lemma}

Before proving \cref{lem:bipartite_antichains}, however, we need the following technical lemma, which shows an upper bound to the number of level-aligned subtrees of a fixed given size.
\begin{lemma}\label{lem:counting_subtrees}
Let $T$ be a complete binary tree of depth $n$. Then the number of its level-aligned subtrees of depth $d$ is upper bounded by $\binom{n}{d+1} \cdot 2^{n2^d}$.
\end{lemma}

The proof of \Cref{lem:counting_subtrees} is a technical calculation and deferred to \Cref{app:additional_proofs:counting_subtrees}.

\begin{proof}[Proof of \cref{lem:bipartite_antichains}]
    For each $r \in R$ we can define a coloring $\chi_r$ of the vertices in $L$ in the following way: $\chi_r(\ell) = \chi(\{\ell,r\})$, i.e.\ $\ell$ is colored as the edge connecting it to $r$. By using~\Cref{prop:php}, since the number of colors is $k$, we obtain the existence of a monochromatic level-aligned subtree $T_r$ of depth $d_L/5k \log k$ for each one of these colorings.
    By~\Cref{lem:counting_subtrees}, the possible number of such level-aligned subtrees in $L$ is $\alpha(d_L,k) = \binom{d_L}{1+d_L/5k\log k} \cdot 2^{d_L 2^{d_L/5k\log k}}$.
    If we consider as colors all possible couples $(S, j)$, where $S$ is a subtree of depth $d_L/5k \log k$ of $L$ and $j \in [k]$, we can define a new coloring over the vertices of $R$ as follows: $r$ is colored $(S, j)$ where $S=T_r$ and $j=\chi_v(\ell)$ is the color of any $\ell \in T_r$, which is well defined because $T_r$ is monochromatic. By using again the pigeonhole principle (in its weaker version) we obtain a monochromatic (weakly embedded) subtree $R'$ of depth at least
    $d_R/k\alpha(d_L,k)$. Let $(L', j)$ be the color assigned to $R'$. The pairs between $L'$ and $R'$ are all colored $j$ by definition of the coloring and by the properties of $L'$, which implies the result.
\end{proof}

Finally, we are ready to formally prove \cref{theorem:bipartite}.

\begin{proof}[Proof of \cref{theorem:bipartite}]
    Let $L$ and $R$ be complete binary trees of depth $n$.
    Denote by $\tilde L$ a subtree of~$L$ of depth $\tilde d$. Apply \cref{lem:bipartite_antichains} on $\tilde L$ and $R$, we obtain a subtree $L'$ of depth 
    \[\frac{\tilde d}{5k\log k}\] 
    and a subtree $R'$ of depth 
    \[\frac{n}{k\alpha(\tilde d,k)},\] 
    where $\alpha(d,k)= \binom{d}{1+d/5k\log k} \cdot 2^{d2^{d/5k\log k}}$.
    Note that if we set $\tilde d =\frac{\log_{(2)}(5n\log k)}{3}$, then the depth of $R'$ is at least the depth of~$L'$. Indeed,
    \begin{align*}
        \mathtt{depth}(R')=\frac{n}{k\alpha(\tilde d,k)}&\geq \frac{n}{ k2^{\tilde d(2^{\tilde d/5k\log k}+1)}} \tag{by the definition of $\alpha(d,k)$}\\
        &\geq\frac{n\tilde d}{k2^{2^{3\tilde d}}} \tag{$\forall d\geq 1$, $d2^{d(2^{ d/5k\log k}+1)}\leq 2^{ d(2^{ d}+2)}\leq 2^{2^{3 d}}$}\\
        &=\frac{\tilde d}{5k\log k} \tag{by the definition of $\tilde d$}\\
        &=\frac{\log_{(2)}(5n\log k)}{15k\log k}=\mathtt{depth}(L').
     \end{align*}

Now, assume that $n\geq 2^{2^{4k}}$ and set $\tilde d=\log_{(3)}(n+1)\cdot 5k\log k$.
    Then, the depth of $R'$ is still at least the depth of $L'$, which is exactly $\log_{(3)}(n+1)$, as wanted. Indeed,
    \begin{align*}
        \mathtt{depth}(R')\geq\frac{n\tilde d}{k2^{2^{3\tilde d}}}\geq\frac{\tilde d}{5k\log k}=\log_{(3)}(n+1)=\mathtt{depth}(L')
    \end{align*}
    if and only if
    \begin{align*}
         n\geq\frac{2^{2^{3\tilde d}}}{5\log k}=\frac{\log (n+1)^{15k\log k}}{5\log k},
    \end{align*}
    which holds since $n\geq 2^{2^{4k}}$.

\end{proof}

\subsection{General $m$-Subsets: Proof of \Cref{thm:ramsey_general_finite}}\label{sec:ramsey_general_proof}

In this section, we prove \Cref{thm:ramsey_general_finite}. 
The proof is based on the same approach as in the proof of \Cref{thm:ramsey_pairs} for the case of pairs. That is, we consider colorings of $m$-subsets of a specific type and prove that, for any such coloring, there exists a monochromatic subtree. Therefore, given an arbitrary coloring of all $m$-subsets, by applying our argument on each type sequentially we obtain a subtree that is colored with at most $\tau(m)$ colors and is also type-monochromatic, in the sense that $m$-subsets of the same type are colored with the same color.
Recall that a pair of vertices $\{u,v\}$ is an $\mathsf{L}$-pair ($\mathsf{R}$-pair) if one between $u$ and $v$ is a left (right) descendant of the other.
Recall that types of $m$-subsets in a tree~$T$ are defined by the following equivalence relation: 
two sets of vertices $A_1, A_2$ in $T$ are equivalent if there is a bijection $\varphi$ from $\bar{A_1}$ to $\bar{A_2}$ such that $\varphi(A_1)=A_2$ and for each $u,v \in \bar{A_1}$,
    \begin{enumerate}
    \centering
        \item $\{u,v\}$ is an $\mathsf{L}$-pair $\Leftrightarrow$ $\{\varphi(u),\varphi(v)\}$ is an $\mathsf{L}$-pair,
        \item $\{u,v\}$ is an $\mathsf{R}$-pair $\Leftrightarrow$ $\{\varphi(u),\varphi(v)\}$ is an $\mathsf{R}$-pair,
    \end{enumerate}
where $\bar{A}$ is the closure of $A$, i.e.\ the minimal subset of vertices containing~$A$ such that for every pair of vertices in $\bar{A}$, their $\LCA$ is also in $\bar{A}$.
The \emph{type} of an $m$-subset is its equivalence class.
Note that by definition all subsets of the same type have the same size.
Consequently, the \emph{size} of a type~$\tau$, denoted by $\lvert \tau \rvert$, is the size of any set of this type.
Observe that any chain is closed (and hence the closure of a chain is itself).
Antichains are not closed and if $m>2$ there exist $m$-antichains that are not equivalent.
(See \Cref{fig:antichains}.)

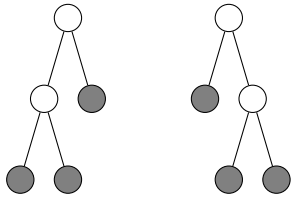
\begin{figure}[thb]
\centering
\begin{adjustbox}{valign=t}
\begin{forest}
rounded/.style={circle, draw, fill opacity=0.5}
    [{}, for tree=rounded
      [{}, 
        [{},fill] 
        [{},fill]
      ]
      [{}, fill]
    ]
\end{forest}\hspace{1cm}
\end{adjustbox}
\begin{adjustbox}{valign=t}
\begin{forest}
rounded/.style={circle, draw, fill opacity=0.5}
    [{}, for tree=rounded
      [{}, fill]
      [{},
        [{},fill] 
        [{},fill]
      ]
    ]
\end{forest}\hspace{1cm}
\end{adjustbox}
\caption{Two $3$-antichains with different closures, hence not equivalent. In general, the number of different closures of $m$-antichains for specific value of $m$, is exactly the number of full binary trees with $m$ leaves, which is known to be the Catalan number $C_{m-1}$.}
\label{fig:antichains}
\end{figure}

\paragraph{Tree-Definable Partitions.} Let $A$ be a set of $m$ vertices in a tree. The $\LCA$ of $A$ induces a partition of the vertices in $A$:
in the case where the $\LCA$ is not in $A$, the partition is into two subsets, one $A_1$
of all vertices in the left subtree of the $\LCA$, and one $A_2$ of all vertices in the right subtree of the $\LCA$.
In the case where the $\LCA$ is in the set $A$, the partition is into three subsets, one $A_1$
of all vertices in the left subtree of the $\LCA$, one $A_2$ of all vertices in the right subtree of the $\LCA$, and one that contains only the $\LCA$.
Observe that any two sets of vertices of the same type admit the same partition in the following sense: let $A\simeq B$ be two subsets of the same type; let $A_1\subseteq A$ be the set of vertices to the left of the $\LCA$ of $A$ and let $B_1\subseteq B$ be the set of vertices to the left of the $\LCA$ of $B$, similarly, let $A_2\subseteq A, B_2\subseteq B$ be the set of vertices to the right of the respective $\LCA$. Then $A_1 \simeq B_1$ and $A_2 \simeq B_2$.
This property is exploited in the proof for the general version of Ramsey theorem for trees in~\Cref{thm:ramsey_general_finite}.
Also, observe that one side in the partition may be empty, but in any case, the sizes of $A_1, A_2$ are both strictly less than $m$. 
This fact is another important ingredient in the proof since it allows us to apply the inductive step on the size of the sets~$m$.

\medskip

The proof uses induction on the size $m$ of the sets, while the induction step follows the same idea as that in the proof of~\Cref{prop:ramsey_pairs_incomp} for the case of incomparable pairs.
In particular, it relies on a version of the bipartite Lemma for sets with size larger than $2$.

\begin{proof}[Proof of \Cref{thm:ramsey_general_finite}]
    The proof is based on the relationships between three types of Ramsey numbers. 
    We start by defining these numbers, then establish the connections between them and finally show their finiteness all together by double induction on the size $m$ of colored sets and on the desired depth $d$ of the subtree.
    
    \paragraph{Ramsey Number.}
    Denote by $\mathsf{R}(d,\tau,k)$ the minimal number $n$ such that for any coloring of all subsets of type $\tau$ in a complete binary tree of depth $n$ with $k$ colors, there exists a monochromatic subtree of depth $d$. 
    \paragraph{Bipartite Ramsey Number.}
    Denote by $\mathsf{R}_b(d, \tau_1, \tau_2, k)$ the minimal number $n$ such that for any complete binary trees $T_1, T_2$ of depth $n$, and any coloring of all pairs $(A_1,A_2)$ where $A_1$ is a subset of type $\tau_1$ in $T_1$ and $A_2$ is a subset of type $\tau_2$ in $T_2$, with $k$ colors, there exist two subtrees $S_1$ of $T_1$ and $S_2$ of $T_2$, of depth $d$ such that all pairs $(A_1,A_2)$ as described above such that $A_1$ in $S_1$ and $A_2$ in $S_2$ are of the same color.
    This number captures a generalization of the bipartite Lemma (\Cref{lem:bipartite_antichains}).
    \paragraph{Local Ramsey Number.}
    Denote by $\mathsf{R_\ell}(d,\tau,k)$ the minimal number $n$ such that for any complete binary tree of depth $n$, and any coloring of all of its subsets of type $\tau$ with $k$ colors, there is a subtree of depth $d$ with the following local Ramsey property: for any vertex $v$ in the subtree, all of the subsets of type $\tau$ in the subtree having $v$ as their $\LCA$, are colored by the same color.

    To simplify the proof we will use the following upper bounds for the above numbers: \\
    \begin{align*}
        \mathsf{R}(d,m,k) &\coloneqq \underset{\lvert \tau \rvert \leq m}{\mathsf{max}} \mathsf{R}(d,\tau,k);\\
        \mathsf{R}_b(d,m,k) &\coloneqq \underset{0 < \lvert\tau_1\lvert,\lvert\tau_2\lvert \leq m}{\mathsf{max}} \mathsf{R}_b(d,\tau_1,\tau_2,k);\\
        \mathsf{R_\ell}(d,m,k) &\coloneqq \underset{\lvert \tau \rvert \leq m}{\mathsf{max}} \mathsf{R_\ell}(d,\tau,k).
    \end{align*}
    For all $d,m,k$ the following holds, provided that the right-hand-sides are finite:

    \begin{enumerate}
        \item $\mathsf{R}(d,m,k) \leq \mathsf{R_\ell}(d\cdot k,m,k)$, \label{item:general_ramsey_1}
        \item $\mathsf{R_\ell}(d,m,k) \leq 1+ \max\Bigl\{\mathsf{R}_b(\mathsf{R_\ell}(d-1,m,k),m-1,k), \mathsf{R}(\mathsf{R_\ell}(d-1,m,k),m-1,k)\Bigr\}$,\label{item:general_ramsey_2}
        \item $\mathsf{R}_b(d,m,k) \leq \mathsf{R}(d,m-1,k\cdot\binom{\mathsf{R}(d,m-1,k)}{d}_T)$,\label{item:general_ramsey_3}
    \end{enumerate}
    where $\binom{n}{d}_T$ stands for the number of subtrees of depth $d$ in a tree of depth $n$.

    \paragraph{Justification for \ref{item:general_ramsey_1}.} Assume that we have a tree $T$ of depth $\mathsf{R_\ell}(d\cdot k,m,k)$ with $k$-coloring of all of $m$-subsets of some type. Then by definition of $\mathsf{R_\ell}$, it admits a subtree of depth $d\cdot k$ with the local Ramsey property as described above. Then we can color each vertex in the subtree, by the promised color from the property, apply the pigeonhole principle for trees on this coloring and get the desired subtree of depth $d\cdot k/k=d$.

    \paragraph{Justification for \ref{item:general_ramsey_2}.} 
    Assume that we have a tree $T$ whose depth is at least the right-hand-side of inequality \ref{item:general_ramsey_2}
    with $k$-coloring of all of $m$-subsets of some type $\tau$. 
   This coloring induces a natural coloring on pairs in $T_1\times T_2$ where $T_1$ is the left subtree of $T$ and $T_2$ is the right subtree of $T$; 
    each colored subset $A$ in $T$, passes its color to the pair of subsets $A_1,A_2\subseteq A$, where $A_i=T_i\cap A$. Notice that by the observation we made before, all the sets $A_1,A_2$ obtained like this have the same pair of types $\tau_1,\tau_2$ of sizes $\lvert \tau_1 \rvert, \lvert \tau_2 \rvert <m$.
    We consider two cases: (i) both $\lvert \tau_1\rvert ,\lvert \tau_2\rvert > 0$.
    Then the left subtree $T_1$ and the right subtree $T_2$ of~$T$ are of depth $\geq \mathsf{R}_b(\mathsf{R_\ell}(d-1,m,k),m-1,k)$. By the definition of $\mathsf{R}_b(\cdot,\cdot,\cdot)$, it is promised that there are subtrees $T'_1, T'_2$ of $T_1,T_2$ accordingly of depth $\mathsf{R_\ell}(d-1,m,k)$ such that all such pairs in $T'_1\times T'_2$ are colored by the same color. Then, while keeping the original coloring to these two subtrees $T'_1, T'_2$, and by the definition of $\mathsf{R_\ell}$, it is promised that there are subtrees  $T''_1, T''_2$ of $T'_1,T'_2$ accordingly of depth $d-1$ admitting the local Ramsey property. Taking these two last subtrees with the root of $T$ gives the desired subtree of depth $d$.
    (ii) In the complementing case either $\lvert \tau_1\rvert = 0$ or $\lvert \tau_2\rvert = 0$.
    Assume without loss of generality that $\lvert \tau_2\rvert = 0$. Note that in this case, 
    the $\LCA$ of $\tau$ belongs\footnote{More precisely, every set of type $\tau$ contains its $\LCA$.} to $\tau$, and hence $\lvert \tau_1\rvert = m-1$. Notice that $T_1$ has depth $\geq \mathsf{R}(\mathsf{R_\ell}(d-1,m,k),m-1,k)$.
    By the definition of $\mathsf{R}(\cdot,\cdot,\cdot)$, it is promised that there is a subtree $T'_1$ of $T_1$ of depth $\mathsf{R_\ell}(d-1,m,k)$ such that all such sets of type $\tau_1$ in $T'_1$ are colored by the same color.
    Pick an arbitrary subtree $T'_2$ of $T_2$ of depth $\mathsf{R_\ell}(d-1,m,k)$. Then, while keeping the original coloring to these two subtrees $T'_1, T'_2$, and by the definition of $\mathsf{R_\ell}$, it is promised that there are subtrees  $T''_1, T''_2$ of $T'_1,T'_2$ accordingly of depth $d-1$ admitting the local Ramsey property. Taking these two last subtrees with the root of $T$ gives the desired subtree of depth $d$.

    \paragraph{Justification for \ref{item:general_ramsey_3}.}
    Assume that we have two trees $T_1,T_2$ of depth $\mathsf{R}\left(d,m-1,k\cdot\binom{\mathsf{R}(d,m-1,k)}{d}_T\right)$ with $k$-coloring of all pairs of $m$-subsets in $T_1\times T_2$ of types $\tau_1,\tau_2$. Take any subtree $T'_2$ of $T_2$ of depth $\mathsf{R}(d,m-1,k)$. Note that $\mathsf{R}(\cdot,\cdot,\cdot)$ is monotonically increasing as a function of the number of colors, so this operation is valid. Let $A_1$ be an $m$-subset in $T_1$ of type $\tau_1$. $A_1$ induces a $k$-coloring on subsets in $T'_2$ of type $\tau_2$; color each $m$-subset $A_2$ in $T'_2$ of this type with the original color of $(A_1,A_2)$. By the definition of $\mathsf{R}$, there is a monochromatic subtree $T''_2$ of $T'_2$ of depth $d$. Now color the set $A_1$ with the pair $(T''_2,c)$ where~$c$ is the color of the subsets in $T''_2$. Apply this process for each such $A_1$ in $T_1$ and get a coloring on $T_1$ with $k\cdot\binom{\mathsf{R}(d,m-1,k)}{d}_T)$ colors. By the definition of $\mathsf{R}$, there is a monochromatic subtree $T'_1$ of $T_1$ of depth $d$. Taking the pair of trees $(T'_1, T''_2)$ gives the desired result since any colored pair of sets $(A_1,A_2)$ in $T'_1\times T''_2$ has the same color.

    \paragraph{Double Induction.} To finish the proof, we show the finiteness of these three Ramsey numbers altogether. 
    We prove by induction on $m$, that $\forall d,k$ $\mathsf{R}(d,m,k),\mathsf{R_\ell}(d,m,k),\mathsf{R}_b(d,m,k) < \infty$. The base case $m=1$ is the pigeonhole principle, and $\mathsf{R}(d,1,k),\mathsf{R}_b(d,1,k),\mathsf{R_\ell}(d,1,k)$ are finite for any values of $d,k$.
   Also observe that for $d<\log m$ the theorem holds trivially for any value of $k$, since any tree of depth $d$ does not contain $m$-subsets for size $m>2^d$.
    For the induction step, suppose that $\mathsf{R}(d,m,k),\mathsf{R}_b(d,m,k),\mathsf{R_\ell}(d,m,k)$ are finite for $m<M$ and any values of $d,k$. From relation~\ref{item:general_ramsey_3} we get that $\mathsf{R}_b(d,M,k)$ is finite for any values of $d,k$. In the next step, we aim to show that $\mathsf{R_\ell}(d,M,k)<\infty$ for any values of $d,k$. 
    Using the finiteness of $\mathsf{R}_b(d,M-1,k),\mathsf{R}(d,M-1,k)$ for any values of $d,k$ together with relation~\ref{item:general_ramsey_2}, we can infer by induction on $d$ that $\mathsf{R_\ell}(d,M,k)<\infty$, since we know for $d<\log M$, so the base case holds. 
    Now, combining this with the first relation~\ref{item:general_ramsey_1}, we can now get that also $\mathsf{R}(d,M,k)<\infty$ for any values of~$d,k$.
    
\end{proof}

\subsection{Chains: Proof of \Cref{thm:ramsey_chains}}\label{sec:ramsey_chains_proof}

In this section, we prove \Cref{thm:ramsey_chains}.
The proof in its core is similar to the proof of the classic Ramsey theorem for sets given by \citet{ErdosRado52}.

Note that the closure of an $m$-chain $C = \{v_1< \ldots < v_m\}$ is $C$ itself, hence the type is determined by the left/right descendant relations between all pairs of vertices in the chain. 
Moreover, since in a chain, the relations among pairs of consecutive vertices determine all other relations, we can exactly represent the type of an $m$-chain $C$ as a tuple $\vec{t}=\vec t(C) \in \{0,1\}^{m-1}$:
that is, $0$ ($1$) indicates the next vertex in the chain is a left (right) descendant.

Denote by $\Ramsey(d,m,k)$ be the smallest $n$ that satisfies the condition in \cref{thm:ramsey_chains}, i.e.\ $\Ramsey(d,m,k)$ is the smallest $n$ such that for every coloring of $m$-chains with $k$ colors, there exists a type-monochromatic subtree of depth $d$. The proof of \cref{thm:ramsey_chains} consists of two parts: the first part shows that $\Ramsey(d,m,k)$ is well defined (i.e.\ $\Ramsey(d,m,k)<\infty$), while the second part proves the quantitative upper bound stated in theorem. The second part is deferred to \cref{app:upper_bound_ramsey_chains}.

\begin{proof}[Proof of \Cref{thm:ramsey_chains}]
    We prove the theorem by induction on $m$.
	The case $m=1$ follows immediately from \Cref{prop:php} since a $1$-chain is simply a vertex, hence a $1$-chain coloring is a vertex coloring, and the promised monochromatic subtree is in particular type-monochromatic.

 \medskip
	
	Assume that the statement holds for $m-1$, and denote ${t=\Ramsey(d,m-1,k^2)}$.  
	Let $T$ be a complete binary tree of depth $n$, where $n$ is sufficiently large (the size of $n$ will be determined later on in \Cref{app:upper_bound_ramsey_chains}).
    Let $\chi$ be an $m$-chain coloring of $T$ using $k$ colors.
    Following the footsteps of the proof for comparable pairs ($m=2$, see \Cref{prop:ramsey_pairs_comp}), we introduce a recursive procedure constructing a subtree of $T$ of depth $t$, denoted $T^\star$.  
    Then, we define an $(m-1)$-chain coloring $\chi^\star$ of $T^\star$, such that any $\chi^\star$-type-monochromatic subtree of $T^\star$ is in fact type-monochromatic with respect to~$\chi$.
    Finally, we apply the induction hypothesis on $T^\star$ and $\chi^\star$, allowing us to obtain the desired type-monochromatic subtree.
    Throughout the proof we denote the vertices of $T^\star$ by $u_\emptyset,u_0,u_1,u_{00},u_{01},\ldots,u_{1^{t}}$, where for a string $\sigma\in\{0,1\}^i$, $i\in\{0,\ldots, t-1\}$, $u_{\sigma 0}$ is the left descendant of $u_\sigma$, and $u_{\sigma 1}$ is the right descendant of $u_\sigma$.

    \medskip
    
    We define by induction a sequence of trees $S_\sigma$ and vertices $u_\sigma$ as follows.
    For every $\sigma\in \{0,1\}^{\leq(m-2)}$, set $u_\sigma$ to be the vertex in $T$ represented by the sequence~$\sigma$ (the root of $T$ is represented by the empty sequence, the left child of the root is $0$, the right child is $1$, etc.).
    Next, for every $\sigma\in \{0,1\}^{m-2}$ set $S_\sigma$ to be the subtree of $T$ rooted at $u_\sigma$. 
    Assume the subtree $S_\sigma$ has been defined and $u_\sigma$ is the root of $S_\sigma$, where $\sigma\in\{0,1\}^i$ is a binary sequence of length~$i\geq m-2$.
    We define subtrees $S_{\sigma b}$ and vertices $u_{\sigma b}$ where $b\in\{0,1\}$, as follows.

    \begin{enumerate}
        \item Consider the $b$'th-subtree of $S_\sigma$, that is the left subtree of $S_\sigma$ for $b=0$ and the right tree of $S_\sigma$ for $b=1$.
        Define an equivalence relation on the vertices of the $b$'th-subtree of $S_\sigma$ as follows:

        \begin{equation*}
            \begin{gathered}
            x\equiv y \\
            \iff \\
            \forall A\subset \{u_{\sigma(0)}, u_{\sigma(1)},\ldots, u_{\sigma(i-1)}\}, |A|=m-2 ~:~ \chi(A\cup\{u_\sigma,x\})=\chi(A\cup\{u_\sigma,y\}),
            \end{gathered}
        \end{equation*}
        where $\sigma(j)$ is the prefix of $\sigma$ of length $j$ (with $\sigma(0)$ being the empty sequence).
        Note that indeed every choice of $m-2$ vertices from the set $\{u_{\sigma(0)}, u_{\sigma(1)},\ldots, u_{\sigma(i-1)}\}$, together with $u_\sigma$ and  a vertex from the $b$'th-subtree of $S_\sigma$, form an $m$-chain in $T$.

        Observe that an equivalence class is determined by a sequence of ${i} \choose {m-2}$ colors, therefore there are at most $k^{ {i} \choose {m-2}}$ such equivalence classes. 

         \item  Apply the pigeonhole principle for trees (\Cref{prop:php}) on the $b$'th-subtree of $S_\sigma$, where the colors of the vertices are the equivalence classes defined in the previous step. Set~$S_{\sigma b}$ to be the promised monochromatic subtree, and set~$u_{\sigma b}$ to be its root.

    \end{enumerate}

     We choose $n$ to be sufficiently large so this procedure may be continued until  $T^\star=\{u_\sigma\}_{\sigma\in\{0,1\}^{\leq t}}$ have been defined. See \Cref{app:upper_bound_ramsey_chains} for a more detailed discussion.
     Note that for every binary sequence $\sigma$ of length $i$, and every $b\in\{0,1\}$,
     \begin{equation}\label{eq:recursive_eq_inside_proof}
         \mathtt{depth}(S_{\sigma b})\geq \left\lfloor\frac{\mathtt{depth}(S_\sigma) -1}{k^{ {i} \choose {m-2}}}\right\rfloor.
     \end{equation}

     Next, define an $(m-1)$-chain coloring of $T^\star$, denoted $\chi^\star$, as follows.
     \[\forall \text{ $(m-1)$-chain $C$ in $T^\star$}~:~ \chi^\star(C)=\bigl(\chi(C\cup \{u_{\sigma_0}\}),\chi(C\cup \{u_{\sigma_1}\})\bigr),\]
    where $u_{\sigma_b}$ is any vertex from $T^\star$ that belongs to the $b$'th subtree emanating from the last vertex of the chain $C$. (If the last vertex in $C$ is at level $t$, meaning it is a leaf of $T^\star$, we just pick an arbitrary color for $C$ out of the $k^2$ possible colors.)
    Note that $\chi^\star$ is well-defined. Indeed, take an $(m-1)$-chain $C$ in $T^\star$. If $x,y$ are vertices that belong to the $b$'th subtree emanating from the last vertex of the chain $C$, then by construction $x\equiv y$, meaning
    $\chi(C\cup\{x\}) = \chi(C\cup\{y\})$.

    \medskip

    Finally, by the induction hypothesis applied on $T^\star$ and $\chi^\star$, and by the choice of $t={\Ramsey(d,m-1,k^2)}$, there exists a type-monochromatic subtree of $T^\star$ of depth~$d$. 
    In fact, this subtree is also type-monochromatic with respect to the original $m$-chain coloring $\chi$. 
    Indeed, if $C$ and $C'$ are two $m$-chains with the same type, then the first $m-1$ vertices of $C$ and $C'$ form an $(m-1)$-chains with the same type that have the same color with respect to $\chi^\star$. So by the definition of $\chi^\star$ it is affirmed that $\chi(C)=\chi(C')$.  
    Therefore, we proved the finiteness of the Ramsey number $\Ramsey(d,m,k)$. It is left to obtain from the recursive procedure described here the upper bound for $\Ramsey(d,m,k)$ that is stated in the theorem. We provide a detailed calculation in \Cref{app:upper_bound_ramsey_chains}.
\end{proof}

\section{Impossibility Result for Private PAC Learning}\label{sec:privacy_vs_LD}

In this section, we prove $\Cref{thm:DP_implies_LD}$. We first start with some notations and definitions that will be useful throughout the proof.

\paragraph{Introducing Notations.}
Let $T$ be a binary mistake tree whose vertices are labeled with instances from a domain $\cX$ and edges are labeled with labels from $\cY$.
Given a learning rule ${\cA:(\cX\times\cY)^\star\to \cY^\cX}$, an input sequence $S$, and an unlabeled example $x\in \cX$, define 
\[\cA_S(x)\coloneqq
\Pr_{h\sim \cA(S)}[h(x)=r_x],\] 
where~$r_x\in \cY$ is the label of the right outgoing edge from $x$ in~$T$.
To ease the notations, from this point onward we assume that the label of every left edge in $T$ is~$0$, and the label of every right edge is~$1$. In particular, by writing $``0"$ we mean the left label, and by writing $``1"$ we mean the right one.\footnote{Note that in the case where $|\cY|>2$, the probabilities of the learner labeling $x$ as the label of a left turn, and as the label of a right turn, does not necessarily sum to $1$. However, every algorithm can be converted to an algorithm for which this sum is $1$ (while maintaining utility and privacy), by a simple post-processing step: if the learner outputs a hypothesis that predicts a label different from the left or right label, replace it with one of them.}

We identify $(m+1)$-chains in $T$ with $T$-realizable samples of size $m$ as follows.
An $(m+1)$-chain  $C=\{x_1<x_2<\ldots<x_{m+1}\}$ in $T$ corresponds to a sample $S=\big((x_1,y_1),\ldots,(x_m,y_m)\big)$, such that $y_i=0$ if $x_{i+1}$ belongs to the left subtree emanating from $x_i$, and $y_i=1$ if $x_{i+1}$ belongs to the right subtree emanating from $x_i$.
Namely, $(y_1,\ldots,y_m)$ is the type\footnote{Recall that the type of an $m$-chain $C$ is described as a tuple $\vec{t}=\vec t(C) \in \{0,1\}^{m-1}$ where $\vec{t}_i=0$ if and only if the $i+1$'th vertex in $C$ is a left descendant of the $i$'th vertex.} of $C$, $\vec t(C)$.

    Let $S=\big((x_1,y_1),\ldots,(x_m,y_m)\big)$ be a sample that is realizable by $T$, and assume for convenience that $S$ is ordered, i.e.\ $x_1<x_2<\ldots<x_m$, where the order is the natural order induced by $T$ . An instance~$x$ is \emph{compatible} with $S$ if there exists a branch in $T$ that realizes $S$ and contains $x$. In such a case we denote by
    \[S^{+x}\coloneqq S\cup(x,y),\]
    where $y\in\{0,1\}$ is a label such that $S\cup(x,y)$ is $T$-realizable:
    note that if $x$ appears earlier in the branch than $x_m$ then there is a unique such $y$.
    In the complementing case, when $x$ appears after $x_m$, both $S\cup(x,0)$ and $S\cup(x,1)$ are realizable by $T$. In that case we arbitrarily pick $S^{+x}=S\cup(x,0)$ and therefore $S^{+x}$ is well defined.
    The \emph{location} of $x$ in $S$ is \[\mathtt{loc}_S(x)\coloneqq \max\{i\mid x_i<x\}.\] 
    If $x<x_1$ then define $\mathtt{loc}_S(x)\coloneqq0$. 

\paragraph{Step 1: Reduction to Approximately Comparison-Based Predictions.}
The first step of the proof hinges on the Ramsey theory we developed for trees and consists of a reduction to subtrees of Littlestone trees on which we can control the output of the algorithm.
To define the notion of \emph{comparison-based predictions}, we first consider deterministic algorithms. 
Given a binary decision tree $T$, a deterministic algorithm is comparison-based with respect to $T$ if its prediction on a test point $x$ depends solely on the labels of the ordered input sample $S$, and the comparisons of the test point with points in $S$, where comparisons are based on the partial order on $T$. A result of comparing two examples $x'$ and $x''$ can be one of five outcomes: (i)+(ii) $x'$ is a left (right) descendant of $x''$, (iii)+(iv) $x''$ is a left (right) descendant of $x'$, or (v) $x'$ and $x''$ are incomparable. We remark that later on we will focus on input samples that are $T$-realizable, and only consider compatible test points, i.e.\ test points that are comparable with every point in the input sample.
For a visual example please refer to \Cref{fig:comparison_based}.

We extend naturally this notion of comparison-based predictions to randomized algorithms. A randomized algorithm is comparison-based with respect to $T$ if for every test point $x$, the probability that it labels $x$ as $1$ is determined by the labels of the ordered input sample $S$, and the comparisons of the test point with points in $S$.
A randomized algorithm is approximately comparison based if its predictions are closed to the prediction of a comparison-based randomized algorithm.
Note that given a $T$-realizable input sequence $S$ and a compatible test point $x$, the labels of $S$ and the comparisons of $x$ and the training points in $S$ are completely encoded by the chain-type of $S^{+x}$, and the location of $x$ inside $S$. 
Therefore, we formally define approximately comparison-based algorithms as follows. 

\begin{definition}[Approximately Comparison-Based Algorithm]\label{def:rand_comparison_alg}
	A (randomized) algorithm~$\cA$, defined over input samples of size $m$, is approximately $\gamma$-comparison-based on $T$ if the following holds.
    There exist numbers~$p_{\vec t,i}\in[0,1]$ for $\vec t\in \{0,1\}^{m+1}$ and $i\in \{0,\ldots,m\}$
    such that for every input sample $S$ of size $m$ realizable by $T$, and for every $x$ compatible with $S$,
    \begin{equation*}\label{eq:comparison_alg_condition}
        \rvert \cA_S(x)-p_{\vec t,i}\lvert \leq \gamma,
    \end{equation*}
    where $\vec t=\vec t(S^{+x})$ is the type of the sample $S^{+x}$, and $i=\mathtt{loc}_S(x)$ is the location of $x$ in $S$.
\end{definition}

As a consequence of the Ramsey theorems we proved for trees, it turns out that every algorithm can be reduced to an approximately comparison-based algorithm.

\begin{lemma}[Every Algorithm is Approx.\ Comparison-based on a Large Subtree]\label{lemma:every_alg_is_comparison_based_somewhere}
    Let
    $\cA$
    be a (possibly randomized) algorithm
    that is defined over input samples of size $m$ over a domain $\cX$, 
    and let $T$ be a decision tree of depth $n$ whose vertices are labeled by instances from $\cX$. Then, there exists a subtree $T'$ of $ T$ of depth $\frac{\log_{(m+1)}(n)}{2^{c2^m m\log m}}$, where $c<35$ is a universal numerical constant, such that $\cA$ is $\left(\frac{1}{100m}\right)$-comparison-based on $T'$.
\end{lemma}

\paragraph{Step 2: Lower Bound the Sample Complexity of Private Approx.\ Comparison-based Algorithms.}
The second step of the proof is providing a lower bound on the sample complexity of private approx.\ comparison-based algorithms that learn a shattered tree $T$.
Note that by \Cref{lemma:reducrion_to_empirical_learner} it is enough to provide a lower bound for algorithms that \emph{empirically} learn $T$, since the sample complexity increases only by a multiplicative constant factor.
We do so by giving a reduction from the \emph{interior point problem}, introduced by \citet*{BunNSV15}. A randomized algorithm solves the interior point problem on~$[n]$ if for every input dataset $X\in[n]^m$, with high probability it returns a point that lies between $\min X$ and $\max X$ (see \Cref{sec:preliminaries_learning}). \citet{BunNSV15} showed that solving the interior point problem with respect to differential privacy requires a dataset size of $m\geq\Omega(\log^\star n)$, we use that bound to derive a bound on the sample complexity.

\medskip The upcoming lemma, in conjunction with \Cref{lemma:every_alg_is_comparison_based_somewhere} from Step 1, implies \Cref{thm:DP_implies_LD}, as we will prove shortly. 

\begin{lemma}\label{lemma:SC_of_CB_alg}[Sample Complexity for Privately Learning Trees]
    Let $T$ be a decision tree of depth $n$ and let $\cA$ be an algorithm defined over input samples of size~$m$. Assume that
    \begin{enumerate}
        \item $\cA$ is $(\epsilon,\delta(m))$-differentially private for some $\epsilon\leq10^{-3}$, and $\delta(m)\leq 1/10^{3}m^2$.
        \item $\cA$ is $\left(\frac{1}{100m}\right)$-comparison-based on $T$.  
        \item \(\mathcal{A}\) is an \((\alpha, \beta)\)-accurate empirical learner for \(T\), where \(\alpha = \beta = 10^{-4}\). Here, by empirical learner for $T$, we mean an empirical learner with respect to input samples that are realizable by (a branch of) \(T\).
    \end{enumerate}
    Then, $m=\Omega(\log^\star n)$.
     
\end{lemma}

\subsection{Proof of \Cref{thm:DP_implies_LD}}
\Cref{thm:DP_implies_LD} follows from \Cref{lemma:SC_of_CB_alg,lemma:every_alg_is_comparison_based_somewhere}. 

\begin{proof}
    Let $\cH$ be a partial concept class over an arbitrary label domain $\cY$, and assume that $\cH$ shatters a Littlestone tree $T$ of depth $d$. Let $\cA$ be any 
    $(\epsilon,\delta(m))$-differentially private learner for~$\cH$, with $\epsilon,\delta(m)$ as in \Cref{thm:DP_implies_LD}.
    We further assume that $\cA$ is an $(\alpha,\beta)$-accurate empirical learner for $\cH$, for  $\alpha=\beta=10^{-4}$.
    This assumption is justified by \Cref{lemma:reducrion_to_empirical_learner}, as the sample complexity of a private empirical learner increases only by a multiplicative constant factor.
    By \Cref{lemma:every_alg_is_comparison_based_somewhere}, there exists a subtree $T'$ of $T$, of depth $\frac{\log_{(m+1)}(d)}{2^{c2^m m\log m}}$ for some universal numerical constant $c<35$, such that $\cA$ is $\left(\frac{1}{100m}\right)$-comparison-based on $T'$.
    Finally, by \Cref{lemma:SC_of_CB_alg} we conclude that 
    \[m\geq \Omega\left(\log^\star\left(\frac{\log_{(m+1)}(d)}{2^{c2^m m\log m}}\right)\right).\]
    Let $t=\log^\star(d)$ and suppose $m\leq \frac{t}{10}$ (else $m=\Omega(\log^\star d)$ and we are done). We claim that $\log^\star\left(\frac{\log_{(m+1)}(d)}{2^{c2^m m\log m}}\right)=\Omega(\log^\star d)$, and therefore $m\geq\Omega(\log^\star(d))$, which concludes the proof.
    Note that by the definition of the $\log^\star$ function, $\twr_{(t)}(1)<d\leq \twr_{(t+1)}(1)$.
    The claim follows from the following calculation:
    \begin{align*}
        \log^\star\left(\frac{\log_{(m+1)}(d)}{2^{c2^m m\log m}}\right)&=1+\log^\star\left(\log_{(m+2)}(d)-c2^m m \log m\right)\tag{definition of $\log^\star$}\\
        &\geq 1+\log^\star\left(\twr_{(t-(m+2))}(1)-c2^m m \log m\right)\tag{$d>\twr_{(t)}(1)$}\\
        &\geq1+\log^\star\left(\twr_{(t/2)}(1)-c2^m m \log m\right)\tag{holds for $t\geq5$ since $m\leq t/10$}\\
        &\geq1+\log^\star\left(\twr_{(t/2)}(1)-2^{t/2}\right)\tag{$\forall m: ~35\cdot2^mm\log m\leq 2^{5m} \leq2^{t/2}$}\\
        &=1+\log^\star\left(\twr_{(t/2)}(1)\right)\tag{holds for $t\geq 10$, see justification below} \\
        &=t/2.\tag{definition of $\log^\star$}
    \end{align*}
Therefore, for large enough $d$, $\log^\star\left(\frac{\log_{(m+1)}(d)}{2^{c2^m m\log m}}\right)\geq \frac{1}{2}\log^\star d$, as desired.
It is left to justify the second-to-last equality. It is enough to show that ${\twr_{(x)}(1)-2^x>\twr_{(x-1)}(1)}$ for $x\geq5$. And indeed $\twr_{(x)}(1)-\twr_{(x-1)}(1)\geq\frac{1}{2}\twr_{(x)}(1)\geq 2^x$ for $x\geq 5$.
\end{proof}

Therefore, it is left to prove \Cref{lemma:SC_of_CB_alg,lemma:every_alg_is_comparison_based_somewhere}.

\subsection{Proof of \Cref{lemma:every_alg_is_comparison_based_somewhere}}

\begin{proof}
    Define a coloring of $(m+2)$-chains of $T$ as follows.
    Let $C=\{x_1<\ldots<x_{m+2}\}$ be an $(m+2)$-chain. 
    Recall that $C$ corresponds to a sample $S=\big((x_1,y_1),\ldots,(x_{m+1},y_{m+1})\big)$ of size $m+1$ where $(y_1,\ldots,y_{m+1})=\vec t(C)$ is the type of $C$.
    For each $i\in\{1,\ldots, m+1\}$, let $S^{-i}$ denote the sample $S\setminus(x_i,y_i)$. 
    Set~$q_i(C)$ to be the fraction of the form $\frac{r}{100m}$ that is closest to $A_{S^{-i}}(x_i)$ (in case of ties pick the smallest such fraction). The color assigned to $C$ is the list $(q_1(C),\ldots,q_{m+1}(C))$.
    \footnote{Notice that the color assigned to \(C\) depends only weakly on the last vertex \(x_{m+2}\) via the label \(y_{m+1}\). We find it more convenient and less cumbersome to increase the size of the chain by one rather than keeping track of the labels.}

    Therefore, the total number of colors is at most $k\coloneqq (100m+1)^{m+1}$. By Ramsey theorem for Chains (\Cref{thm:ramsey_chains}) there exists a subtree $T'$ that is type-monochromatic with respect to the above coloring of depth
    \begin{align*}
        d\geq&\frac{\log_{(m+1)}(n)}{5 \cdot 2^m k^{2^{m+1}}\log k} \\
        =&\frac{\log_{(m+1)}(n)}{5\cdot 2^m (100m+1)^{(m+1)2^{m+1}}(m+1)\log(100m+1)}\\
        =&\frac{\log_{(m+1)}(n)}{2^{\log 5+m+\log(100m+1)(m+1)2^{m+1}+\log(m+1)+\log\log(100m+1)}}\\
        \geq&\frac{\log_{(m+1)}(n)}{2^{c2^m m\log m}},
    \end{align*}
    where $1<c<35$ is a universal numerical constant.
    For every possible type $\vec t\in\{0,1\}^{m+1}$ and $i\in\{0,\ldots,m\}$, set $p_{\vec t,i}$ to be $q_{i+1}(C)$ where $C$ is any $\vec t$-typed $(m+2)$-chain in $T'$. Note that $p_{\vec t,i}$ is well defined since $T'$ is type-monochromatic.
    It is straightforward to verify that $\cA$ is $\left(\frac{1}{100m}\right)$-comparison-based on $T'$ with respect to $\{p_{\vec t,i}\}$, as wanted. 
\end{proof}

\subsection{Proof of \Cref{lemma:SC_of_CB_alg}}\label{sec:proof_of_reduction_to_ipp_lemma}

The proof of \Cref{lemma:SC_of_CB_alg} 
hinges on a reduction from the interior point problem. Before describing the reduction, we will introduce notation which will be used in this proof.
Given a branch $B$ in a tree and an example $z$ on $B$, denote
    \[b_z=\begin{cases}
        0, & \text{if } (z,0)\in B\\
        1, & \text{if } (z,1)\in B
    \end{cases}.\]

\paragraph{Reduction from Interior Point Problem.}
Let $T$ be a tree of depth $n$, and $\cA$ be an algorithm as in \Cref{lemma:SC_of_CB_alg}. Let $d_1\ldots d_{m}\in [n]$ be natural numbers, the input to the interior point problem. For convenience, assume that they are ordered $d_1\leq \ldots\leq d_m$. Additionally, assume that $d_{i+1}-d_i>\log^2 n$ for all $1\leq i<m$.
Define algorithm $\tilde{\cA}$ as follows.

\renewcommand{\algorithmicrequire}{\textbf{Input:}}
\renewcommand{\algorithmicensure}{\textbf{Output:}}

\begin{algorithm}[H]
\caption{$\tilde \cA$ (Reduction from IPP)}\label{alg:ipp}
\begin{algorithmic}
\Require $d_1\ldots,d_m$.
\State - Sample uniformly at random a branch $B\sim\mathtt{Branches}(T)$.
\State - $S \gets \left((x_1,y_1),\ldots (x_m,y_m)\right)$, where $x_i$ is the point of depth $d_i$ in $B$, and $y_i=b_{x_i}$.
\State - Sample $h\sim\cA(S)$.
\State - Search for ``long almost-correct intervals": a long almost-correct interval is a sequence of consecutive examples $Z=(z_1,\ldots z_l)$ on $B$ of length $l= \lfloor\log ^2 n\rfloor$, such that $\sum_{i=1}^l\1[h(z_i)=b_{z_i}]\geq 0.9\cdot l$, where $n=\mathtt{depth}(T)$.
\Ensure Output $\max \big\{\mathtt{depth}(z_1)\mid Z=(z_1\ldots z_l) \text{ is a long almost-correct interval}\big\}$. 

In other words, output the depth of the first point of the deepest almost-correct interval. If there are no long almost-correct intervals, return $n$.
\end{algorithmic}
\end{algorithm}

\begin{claim}\label{claim:reduction}
    Assuming $\cA$ is an algorithm as in \Cref{lemma:SC_of_CB_alg}, $\tilde \cA$ is $(\epsilon,\delta(m))$-differentially private, and with probability at least $\frac{3}{4}$ its output lies between $d_1$ and $d_m$.
\end{claim}

\Cref{lemma:SC_of_CB_alg} is a direct corollary of \Cref{claim:reduction,lemma:rescaling_ipp}.
In order to prove \Cref{claim:reduction}, we first need to collect some lemmas and definitions.

\begin{definition}\label{def:good_pair}
    Given a pair of consecutive examples $(x_i, x_{i+1})$ in $S$, we define the following properties.
    \begin{enumerate}
        \item \textbf{``Sign change":} The pair is considered a sign-changing pair if $y_i\neq y_{i+1}$.
        \item \textbf{``Matching neighbors":} The pair is considered to have matching-label neighbors if there exist examples $x'$ and $x''$ compatible with $B$ such that\begin{align*}
            x_{i-1}<x'<x_i &\text{ and } b_{x'}=y_i,\\
            x_{i+1}<x''<x_{i+2} &\text{ and } b_{x''}=y_{i+1}.
        \end{align*} 
        In words, there exists an example before $x_i$ with the same label as $x_i$, and there exists an example after $x_{i+1}$ with the same label as $x_{i+1}$, and in addition, these examples located in the $i-1$ and $i+1$ locations with respect to $S$, respectively.
        \item \textbf{``Correct":} The pair is considered correct if
        \begin{align*}
            |A_S(x_i)-y_i|&\leq \xi, \\
            |A_S(x_{i+1})-y_{i+1}|&\leq \xi,
        \end{align*}
        for $\xi=17(\alpha+\beta)=34\cdot10^{-4}$.
       In other words, a random hypothesis sampled from $\cA(S)$ labels correctly $x_i$ with probability at least $1-\xi$, and labels correctly $x_{i+1}$ with probability at least~$1-\xi$.
    \end{enumerate}
\end{definition}

We will show that with high probability over the randomness of $\tilde \cA$, there exists a pair of examples $(x_i,x_{i+1})$ that satisfies the above properties. Then, we will show that since $\cA$ is private and approximately comparison-based, the interval between $x_i$ and $x_{i+1}$ is a long almost-correct interval (refer to \Cref{alg:ipp} for the definition of a 'long almost-correct interval.').

\begin{lemma}\label{lemma:good_pair}
    There exists an index $i$ such that the pair of examples $(x_i,x_{i+1})$
    satisfies properties $1,2$ and $3$ as in \Cref{def:good_pair},
    with probability $\geq 1-2\exp{\left(-\frac{m-3}{16}\right)}$ where the probability is taken over the randomness of $\tilde \cA$ (i.e.\ over the choice of the random branch).
\end{lemma}

\begin{proof} 
    For $1< j < m-1$, we define the following good events:
    \begin{align*}
        \mathsf{SC}_j &~: &&\text{The pair $(x_j,x_{j+1})$ satisfies property $1$ (sign change).}\\
        \mathsf{MN}_j &~: &&\text{The pair $(x_j,x_{j+1})$ satisfies property $2$ (matching neighbors), and in addition}\\
         & &&\mathtt{dist}_T(x',x_j),\mathtt{dist}_T(x_{j+1},x'')\leq(\log^2 n) / 2,
    \end{align*}
    where the distance between two comparable examples $z_1,z_2$ in $T$ is $\mathtt{dist}_T(z_1,z_2)\coloneqq{|\mathtt{depth}(z_1)-\mathtt{depth}(z_2)|}$.
    Next, define random variables $X_j=\1[\mathsf{SC}_j\cap \mathsf{MN}_j]$ for $1< j < m-1$.
    Note that: 
    \begin{itemize} 
        \item [(a)] $X_j$ and $X_{j'}$ are independent if $|j-j'|>1$, because the distance between every consecutive examples $x_j,x_{j+1}$ in $S$ is at least $\log^2 n$.
        In particular, $\{X_{2k}\}_{k=1}^{\lfloor(m-2)/2\rfloor}$ are IID.
        \item[(b)] Denote $D=\lfloor(\log^2 n)/2\rfloor$. Then, since the branch $B$ is sampled uniformly at random,
        \begin{align*} 
            \bbE[X_j]&=1-\left(\Pr[(\mathsf{SC}_j)^\complement]+\Pr[\mathsf{SC}_j]\cdot\Pr[(\mathsf{MN}_j)^\complement]\right) \tag{ $\mathsf{SC}_j,\mathsf{MN}_j$ are independent}\\
            &=1-\left(\frac{1}{2}+\frac{1}{2}\left(2 \cdot 2^{-D}-2^{-2D}\right)\right)\\
            &=\frac{1}{2}-2^{-D}+2^{-(2D+1)}\\
        \end{align*} 
    \end{itemize}
    
    Let $X=\sum_{k=1}^{\lfloor\frac{m-2}{2}\rfloor} X_{2k}$. By Chernoff,
    \begin{align*}
        \Pr\left[\left|X- {\left\lfloor\frac{m-2}{2}\right\rfloor}\cdot \bbE[X_j]\right|\geq \left\lfloor\frac{m-2}{2}\right\rfloor\cdot \frac{1}{4}\right]&\leq 2 \exp{\left(-\frac{1}{8}\left\lfloor\frac{m-2}{2}\right\rfloor\right)}\\
        &\leq 2 \exp{\left(-\frac{m-3}{16}\right)}.
    \end{align*}
    Note that 
    \begin{align*}
        \Pr\left[X> \left\lfloor\frac{m-2}{2}\right\rfloor\cdot 
        \underbrace{\left(\frac{1}{2}-\frac{1}{2^D}+\frac{1}{2^{2D+1}}\right)}_{\bbE[X_j]}
        -\left\lfloor\frac{m-2}{2}\right\rfloor\cdot\frac{1}{4}\right]&\leq
        \Pr\left[X > \frac{m-2}{2}\cdot\frac{3}{8}-\frac{m-3}{2}\cdot \frac{1}{4}\right]\\
        &=\Pr\left[X > \frac{m}{16}\right],
    \end{align*}
    where the inequality holds if $\bbE[X_j]\geq 3/8$, which holds for $n\geq 8$.
Therefore, with probability at least $1-2\exp{\left(-\frac{m-3}{16}\right)}$,
$X> m/16$, i.e.\ the number of pairs $(x_j,x_{j+1})$ that satisfy properties~$1$ and~$2$ is at least $m/16$.

Consider the case when there exist at least $m/16$ pairs satisfying properties~$1$ and~$2$, and assume towards contradiction that all of the pairs satisfying properties~$1$ and~$2$ do not satisfy property~$3$. I.e., if~$(x_j,x_{j+1})$ is a sign-changing pair with matching-label neighbors, then either ${|A_S(x_j)-y_j|> \xi}$, or ${|A_S(x_{j+1})-y_{j+1}|> \xi}$.
    Note that since $\cA$ is an $(\alpha,\beta)$-empirical learner,
    \begin{align*}
        \bbE_{h\sim\cA(S)}[\loss{S}{h}]=\frac{1}{m}\sum_{i=1}^m \Pr_{h\sim\cA(S)}[h(x_i)\neq y_i]=\frac{1}{m}\sum_{i=1}^m |\cA_S(x_i)-y_i|\leq\alpha+\beta,
    \end{align*}
    and by the assumption above it holds that
    \begin{align*}
        \frac{1}{m}\sum_{i=1}^m |\cA_S(x_i)-y_i|>\frac{1}{16}\xi.
    \end{align*}
    But this is a contradiction to the choice of $\xi=34\cdot 10^{-4},\alpha=10^{-4},\beta=10^{-4}$.
    Therefore, with probability at least $1-2\exp{\left(-\frac{m-3}{16}\right)}$,
    there exists $1<i<m-1$ such that the pair of examples $(x_i,x_{i+1})$ satisfy properties $1,2$ and $3$ as wanted.

\end{proof}

\begin{definition}
    A pair of consecutive points $(x_i,x_{i+1})$ on $S$ is $\cA$-good with parameter $\xi'$ if the following holds.
    For every compatible point $x$ between~$x_i$ and~$x_{i+1}$ (i.e.\ $x_i<x<x_{i+1}$ under the partial order induced by $T$), 
    \[|\cA_S(x)-b_x|\leq \xi'.\]
\end{definition}

\begin{lemma}\label{lemma:A_good_pair}
    There exists a pair of consecutive points $(x_i,x_{i+1})$ in $S$ that is $\cA$-good with parameter $\xi'=2(e^\epsilon-1+\delta(m))+\frac{2}{100m}+18(\alpha+\beta)$, with probability~${\geq 1-2\exp{\left(-\frac{m-3}{16}\right)}}$ over the randomness of~$\tilde\cA$.
\end{lemma}
\begin{proof}
    By \Cref{lemma:good_pair}, with high probability at least $1-2\exp{\left(-\frac{m-3}{16}\right)}$ over the choice of the sample~$S$,
    there exists a pair of examples $(x_i,x_{i+1})$ such that
    \begin{enumerate}
        \item $y_i\neq y_{i+1}$;
        \item there exist $x_{i-1}<x'<x_i$ and $x_{i+1}<x''<x_{i+2}$, such that $b_{x'}=y_i$ and $b_{x''}=y_{i+1}$;
        \item $|\cA_S(x_i)-y_i|\leq \xi$ and $|\cA_S(x_{i+1})-y_{i+1}|\leq \xi$.\label{eq:correct}
    \end{enumerate}
    It suffices to show that conditioned on this event there exists a $\cA$-good consecutive pair with parameter $\xi'$.
    Let $x$ be a compatible example between $x_i$ and $x_{i+1}$ such that~$b_{x}=y_i$.
    Denote by $\tilde S$ the sample that is obtained by replacing $(x_i,y_i)$ in $S$ with $(x',y_i)$, where $x'$ is the matching-label neighbor of~$x_i$.
    Observe that
    \begin{align*}
        \vec t (\tilde S^{+x})=\vec t(\tilde S^{+x_i}), \text{ and}\\
        \mathtt{loc}_{\tilde S}(x)=\mathtt{loc}_{\tilde S}(x_i)=i,
    \end{align*}
    where $\vec t(\tilde S^{+z})$ is the type of $\tilde S$ together with $z$, and $\mathtt{loc}_{\tilde S}(z)$ is the location of $z$ inside~$\tilde S$.
    See \Cref{fig:4} for an illustration.
    Therefore, since $\cA$ is a $\left(\frac{1}{100m}\right)$-comparison-based algorithm on $T$,
    \begin{equation}\label{eq:C-B}
        |\cA_{\tilde S}(x)-\cA_{\tilde S}(x_i)|\leq \frac{2}{100m}.
    \end{equation}

\begin{figure}[!h]
\centering
\begin{tikzpicture}
\begin{scope}[rotate around={90:(0,0)}]
\foreach \x in {-2,...,5}{
  \draw  (0+\x*1,0) -- ++(0.5,1) -- ++ (0.5, -1);
}
\end{scope}
\fill [color={blue}] (-0.95,5.5) circle (0.1cm);
\node [font=\normalsize] at (-1.4,5.6) {$x'$};
\draw [ color={cyan}, line width=2pt, ->] (-0.95,5.5) -- (-0.3,5.15);

\fill [color={blue}] (-0.95,3.5) circle (0.1cm);
\node [font=\normalsize] at (-1.4,3.6) {$x_i$};
\draw [ color={cyan}, line width=2pt, ->] (-0.95,3.5) -- (-0.3,3.15);

\fill [color={blue}] (-0.95,1.5) circle (0.1cm);
\node [font=\normalsize] at (-1.4,1.6) {$x$};
\draw [ color={cyan}, line width=2pt, ->] (-0.95,1.5) -- (-0.3,1.15);

\fill [color={blue}] (0,-1) circle (0.1cm);
\node [font=\normalsize] at (0.8,-1.05) {$x_{i+1}$};

\node [font=\normalsize, color={purple}] at (1.7,5.55) {$i$'th example in $\tilde{S}$};
\draw [ color={purple}, line width=1pt, ->] (0.25,5.5) -- (-0.5,5.5);

\node [font=\normalsize, color={purple}] at (3.2,-1) {$(i+1)$'th example in $\tilde{S}$};
\draw [ color={purple}, line width=1pt, ->] (0.25,5.5) -- (-0.5,5.5);

\end{tikzpicture}
\caption{Both $x$ and $x_i$ lies between the $i$th and the $(i+1)$th examples of $\tilde S$, and both $(x_i,x_{i+1})$ and $(x,x_{i+1})$ are $y_i$-pairs in $T$.}
\label{fig:4}
\end{figure}
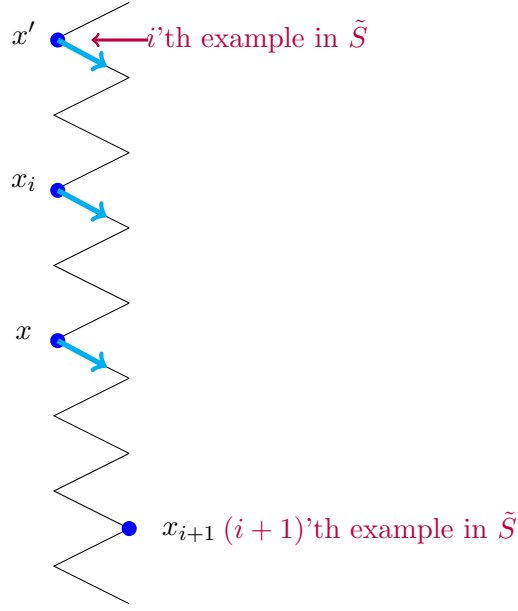

    Moreover, since $\cA$ is $(\epsilon,\delta)$-differentially private, 
    $\cA_{S}(x)\overset{\epsilon,\delta}{\approx}\cA_{\tilde S}(x)$ and $\cA_{S}(x_i)\overset{\epsilon,\delta}{\approx}\cA_{\tilde S}(x_i)$.
    Note that for every $\epsilon,\delta\geq 0$, if $a,b\in[0,1]$ satisfy $a\overset{\epsilon,\delta}{\approx}b$, then $|a-b|\leq e^\epsilon-1+\delta$. Indeed,
    \begin{align*}
            a-b\leq (e^\epsilon-1)\cdot b +\delta\leq e^\epsilon-1+\delta,\\
            b-a\leq (e^\epsilon-1)\cdot a +\delta\leq e^\epsilon-1+\delta.
        \end{align*}

    Therefore,
    \begin{align}
        |\cA_{S}(x)-\cA_{\tilde S}(x)|\leq e^\epsilon-1+\delta,\label{eq:DP1}\\ 
        |\cA_{S}(x_i)-\cA_{\tilde S}(x_i)|\leq e^\epsilon-1+\delta\label{eq:DP2}.
    \end{align}

    Finally, by \Cref{eq:C-B,eq:DP1,eq:DP2,eq:correct}
    \begin{align*}
        |\cA_S(x)-y_i|&\leq |\cA_S(x)-\cA_{\tilde S}(x)| + |\cA_{\tilde S}(x)-\cA_{\tilde S}(x_i)| + |\cA_{\tilde S}(x_i)-\cA_S(x_i)| + |\cA_S(x_i)-y_i|\\
        &\leq 2(e^\epsilon-1+\delta)+\frac{2}{100m}+\xi\\
        &\leq \xi',
    \end{align*}
    by the choice of $\xi'$.
    
    The case $b_{x}=y_{i+1}$ is symmetric, therefore $(x_i,x_{i+1})$ is $\cA$-good pair as wanted.
\end{proof}

\begin{lemma}\label{lemma:deep_sequences_are_not_long_almost_correct}
   Let $l=\lfloor\log^2 n\rfloor$ and let $Z=(z_1<z_2<\ldots<z_l)$ be a sequence of consecutive examples on~$B$. Assume that $Z$ starts below $S$ on $T$, i.e.\ $x_{m}<z_1$. Then, 
   \[\Pr_{B,h\sim \cA(S)}\left[ \sum_{i=1}^l \1[h(z_i)=b_{z_i}] \geq 0.9\cdot l\right]\leq \exp{\left(-2\cdot\frac{4}{25}\cdot l\right)}.\]
\end{lemma}

\begin{proof}
    We can conceptualize the randomness of the reduction algorithm $\tilde \cA$
    in the following manner: Initially, a fair coin is independently tossed $d_m + 1$ times (recall that $d_m$ is the largest input). These coin tosses determine the first $d_m + 1$ turns of the branch $B$, which in turn determines the sample $S$, the input for $\cA$. Subsequently, the coin is independently tossed $n - d_m - 1$ more times, completing the determination of the branch $B$.
    This illustrates that, for every example $z$ that is below $S$, $A_S(z)$ is independent of $b_z$, and  $\Pr[h(z)=b_{z}]=1/2$. Therefore, $\bbE\left[\sum_{i=1}^l \1[h(z_i)=b_{z_i}]\right]=l/2$,
    and the statement follows by applying a standard Chernoff bound.
\end{proof}

We turn to prove \Cref{claim:reduction}.

\begin{proof}[Proof of \Cref{claim:reduction}]
    
    We start by showing that $\tilde \cA$ is $(\epsilon,\delta(m))$-differentially private, and then we will show that with probability at least $3/4$ its output is an interior point.
    
    \paragraph{Privacy:}
    Let \(D\) and \(D'\) be neighboring datasets. Consider the output distributions \(\tilde{\mathcal{A}}(D)\) and \(\tilde{\mathcal{A}}(D')\). We couple these distributions by selecting the same random branch \(B\) in the first step of the algorithm. Hence, the samples \(S\) and \(S'\) that are input to \(\mathcal{A}\) are also neighbors. Since \(\mathcal{A}\) is \((\epsilon, \delta)\)-DP, it follows that the distributions \(\mathcal{A}(S)\) and \(\mathcal{A}(S')\) are \((\epsilon, \delta)\)-indistinguishable. Since the outputs \(\tilde{\mathcal{A}}(D)\) and \(\tilde{\mathcal{A}}(D')\) are functions of \(\mathcal{A}(S)\) and \(\mathcal{A}(S')\), by post-processing (\Cref{prop:dp_post_processing}), it follows that \(\tilde{\mathcal{A}}(D)\) and \(\tilde{\mathcal{A}}(D')\) are also \((\epsilon, \delta)\)-indistinguishable. Hence, \(\tilde{\mathcal{A}}\) is \((\epsilon, \delta)\)-DP.

    \paragraph{Utility:}
    Observe the following.
    \begin{enumerate}
        \item  If $(x_i,x_{i+1})$ is an $\cA$-good pair with parameter $\xi'$, 
        then for any sequence of consecutive examples $Z=(z_1,\ldots,z_l)$ of length $l=\lfloor \log^2 n\rfloor$, such that $x_i\leq z_1<z_l\leq x_{i+1}$, it holds that
    \[\bbE_{B,h\sim \cA(S)} \left[\sum_{i=1}^l\1[h(z_i)\neq b_{z_i}]~\Big|~ \text{$(x_i,x_{i+1})$ is $\cA$-good with parameter $\xi'$}\right]=\sum_{i=1}^l|\cA_S(z_i)-b_{z_i}|\leq \xi'\cdot l.\]
    By Markov's inequality,
    \[\Pr\left[\sum_{i=1}^l\1[h(z_i)\neq b_{z_i}]\geq 0.1\cdot l\right]\leq 10\cdot \xi'.\]
    \item By \Cref{lemma:A_good_pair} there are no $\cA$-good pairs with probability $\leq 2\exp{\left(-\frac{m-3}{16}\right)}$.
    \item By \Cref{lemma:deep_sequences_are_not_long_almost_correct},
    \begin{align*}
        \Pr\left[  \text{$\exists$ a long almost-correct $Z=(z_1,\ldots, z_l)$  that starts below $S$}  \right]\leq n\cdot \exp{\left(-\frac{8}{25}\log^2 n\right)}
    \end{align*}
    \end{enumerate}
    Therefore, by applying a union bound,
    \begin{align*}
        \Pr\left[\substack{\text{$\tilde \cA$ returns a point that is}\\\text{  not between $d_1$ and $d_m$}}\right]&\leq 10\xi'+ 2\exp{\left(-\frac{m-3}{16}\right)} + n\cdot \exp{\left(-\frac{8}{25}\log^2 n\right)}\\
        &= 20(e^\epsilon-1+\delta(m))+\frac{20}{100m}+180(\alpha+\beta)+2\exp{\left(-\frac{m-3}{16}\right)} + \\
        &+n\cdot \exp{\left(-\frac{8}{25}\log^2 n\right)}.
    \end{align*}
    Each one of the five summands is smaller than $1/20$, by the choice of $\epsilon=10^{-3},\delta(m)=1/10^3m, \alpha=10^{-4}, \beta=10^{-4}$, and for large enough\footnote{Note that, without loss of generality, we may assume that $m\geq 1/\alpha=10^4$ since $\cA$ is an $(\alpha, \beta)$-accurate learner for $T$. Additionally, the assumption that $n$ is large enough is concealed in the big $\Omega$ notation.} $m$ and $n$.
    Therefore, the claim follows.
    \end{proof}

    \section*{Acknowledgements}
    We thank Noga Alon, Zachary Chase, Roi Livni, and Uri Stemmer, for insightful discussions.

    SM is a Robert J.\ Shillman Fellow; he acknowledges support by ISF grant 1225/20, by BSF grant 2018385, by an Azrieli Faculty Fellowship, by Israel PBC-VATAT, by the Technion Center for Machine Learning and Intelligent Systems (MLIS), and by the the European Union (ERC, GENERALIZATION, 101039692). Views and opinions expressed are however those of the author(s) only and do not necessarily reflect those of the European Union or the European Research Council Executive Agency. Neither the European Union nor the granting authority can be held responsible for them.

    \bibliographystyle{plainnat}
    \bibliography{references}

\begin{thebibliography}{56}
\providecommand{\natexlab}[1]{#1}
\providecommand{\url}[1]{\texttt{#1}}
\expandafter\ifx\csname urlstyle\endcsname\relax
  \providecommand{\doi}[1]{doi: #1}\else
  \providecommand{\doi}{doi: \begingroup \urlstyle{rm}\Url}\fi

\bibitem[Aden-Ali et~al.(2023)Aden-Ali, Cherapanamjeri, Shetty, and Zhivotovskiy]{Ali23}
Ishaq Aden-Ali, Yeshwanth Cherapanamjeri, Abhishek Shetty, and Nikita Zhivotovskiy.
\newblock Optimal pac bounds without uniform convergence.
\newblock In \emph{Proc.\ 64th Symp.\ Foundations of Computer Science (FOCS)}, pages 1203--1223, 2023.

\bibitem[Alon and Maass(1986)]{AlonM86}
Noga Alon and Wolfgang Maass.
\newblock Meanders, ramsey theory and lower bounds for branching programs.
\newblock In \emph{Proc.\ 27th Symp.\ Foundations of Computer Science (FOCS)}, pages 410--417, 1986.

\bibitem[Alon and Spencer(2008)]{AlonBook}
Noga Alon and Joel~H. Spencer.
\newblock \emph{The Probabilistic Method, Third Edition}.
\newblock Wiley-Interscience series in discrete mathematics and optimization. Wiley, 2008.

\bibitem[Alon et~al.(2019)Alon, Livni, Malliaris, and Moran]{AlonLMM19}
Noga Alon, Roi Livni, Maryanthe Malliaris, and Shay Moran.
\newblock Private {PAC} learning implies finite littlestone dimension.
\newblock In \emph{Proc.\ 51st Symp.\ Theory of Computing (STOC)}, pages 852--860, 2019.

\bibitem[Alon et~al.(2021)Alon, Hanneke, Holzman, and Moran]{AlonHHM21}
Noga Alon, Steve Hanneke, Ron Holzman, and Shay Moran.
\newblock A theory of {PAC} learnability of partial concept classes.
\newblock In \emph{Proc.\ 62nd Symp.\ Foundations of Computer Science (FOCS)}, pages 658--671, 2021.

\bibitem[Alon et~al.(2022)Alon, Bun, Livni, Malliaris, and Moran]{AlonBLMM22}
Noga Alon, Mark Bun, Roi Livni, Maryanthe Malliaris, and Shay Moran.
\newblock Private and online learnability are equivalent.
\newblock \emph{J. {ACM}}, 69\penalty0 (4):\penalty0 28:1--28:34, 2022.

\bibitem[Alon et~al.(2023)Alon, Moran, Schefler, and Yehudayoff]{alon2023unified}
Noga Alon, Shay Moran, Hilla Schefler, and Amir Yehudayoff.
\newblock A unified characterization of private learnability via graph theory, 2023.

\bibitem[Barak et~al.(2010)Barak, Kindler, Shaltiel, Sudakov, and Wigderson]{BarakKSSW10}
Boaz Barak, Guy Kindler, Ronen Shaltiel, Benny Sudakov, and Avi Wigderson.
\newblock Simulating independence: New constructions of condensers, ramsey graphs, dispersers, and extractors.
\newblock \emph{J. {ACM}}, 57\penalty0 (4):\penalty0 20:1--20:52, 2010.

\bibitem[B{\'{a}}r{\'{a}}ny and K{\'{a}}rolyi(2000)]{Barany2000}
Imre B{\'{a}}r{\'{a}}ny and Gyula K{\'{a}}rolyi.
\newblock Problems and results around the erd{\"{o}}s-szekeres convex polygon theorem.
\newblock In \emph{Proc.\ Japanese Conf., Discrete and Computational Geometry (JCDCG)}, pages 91--105, 2000.

\bibitem[Beimel et~al.(2013)Beimel, Nissim, and Stemmer]{beimel2013characterizing}
Amos Beimel, Kobbi Nissim, and Uri Stemmer.
\newblock Characterizing the sample complexity of private learners.
\newblock In \emph{ITCS}. ACM, 2013.

\bibitem[Beimel et~al.(2019)Beimel, Nissim, and Stemmer]{Beimel19Pure}
Amos Beimel, Kobbi Nissim, and Uri Stemmer.
\newblock Characterizing the sample complexity of pure private learners.
\newblock \emph{Journal of Machine Learning Research}, 20\penalty0 (146):\penalty0 1--33, 2019.
\newblock URL \url{http://jmlr.org/papers/v20/18-269.html}.

\bibitem[Ben{-}David et~al.(2009)Ben{-}David, P{\'{a}}l, and Shalev{-}Shwartz]{Ben-DavidPS09}
Shai Ben{-}David, D{\'{a}}vid P{\'{a}}l, and Shai Shalev{-}Shwartz.
\newblock Agnostic online learning.
\newblock In \emph{Proc.\ 22nd Conf.\ Learning Theory (COLT)}, 2009.

\bibitem[Brukhim et~al.(2022)Brukhim, Carmon, Dinur, Moran, and Yehudayoff]{brukhim2022characterization}
Nataly Brukhim, Daniel Carmon, Irit Dinur, Shay Moran, and Amir Yehudayoff.
\newblock A characterization of multiclass learnability.
\newblock In \emph{Proc.\ 63rd Symp.\ Foundations of Computer Science (FOCS)}, pages 943--955. IEEE, 2022.

\bibitem[Bun(2016)]{BunThesis}
Mark Bun.
\newblock \emph{New Separations in the Complexity of Differential Privacy.}
\newblock PhD thesis, Harvard University, Graduate School of Arts \& Sciences, 2016.

\bibitem[Bun et~al.(2015)Bun, Nissim, Stemmer, and Vadhan]{BunNSV15}
Mark Bun, Kobbi Nissim, Uri Stemmer, and Salil~P. Vadhan.
\newblock Differentially private release and learning of threshold functions.
\newblock In \emph{Proc.\ 56th Symp.\ Foundations of Computer Science (FOCS)}, pages 634--649, 2015.

\bibitem[Bun et~al.(2020)Bun, Livni, and Moran]{BunLM20}
Mark Bun, Roi Livni, and Shay Moran.
\newblock An equivalence between private classification and online prediction.
\newblock In Sandy Irani, editor, \emph{61st {IEEE} Annual Symposium on Foundations of Computer Science, {FOCS} 2020, Durham, NC, USA, November 16-19, 2020}, pages 389--402. {IEEE}, 2020.
\newblock \doi{10.1109/FOCS46700.2020.00044}.
\newblock URL \url{https://doi.org/10.1109/FOCS46700.2020.00044}.

\bibitem[Cai and Yan(2019)]{CountingBorel19}
Yue Cai and Catherine Yan.
\newblock Counting with borel’s triangle.
\newblock \emph{Discrete Mathematics}, 342\penalty0 (2):\penalty0 529--539, 2019.

\bibitem[Chattopadhyay and Zuckerman(2016)]{Eshan16}
Eshan Chattopadhyay and David Zuckerman.
\newblock Explicit two-source extractors and resilient functions.
\newblock In \emph{Proc.\ 48th Symp.\ Theory of Computing (STOC)}, page 670–683, 2016.

\bibitem[Cheung et~al.(2023)Cheung, Hatami, Hatami, and Hosseini]{CheungHHH23}
TsunMing Cheung, Hamed Hatami, Pooya Hatami, and Kaave Hosseini.
\newblock Online learning and disambiguations of partial concept classes.
\newblock In \emph{Proc.\ 50th Intl.\ Coll.\ Autom.\ Lang.\ Program.\ (ICALP)}, volume 261, pages 42:1--42:13, 2023.

\bibitem[Conlon et~al.(2015)Conlon, Fox, and Sudakov]{ConlonFS15}
David Conlon, Jacob Fox, and Benny Sudakov.
\newblock Recent developments in graph ramsey theory.
\newblock In \emph{Surveys in Combinatorics 2015}, volume 424 of \emph{London Mathematical Society Lecture Note Series}, pages 49--118. 2015.

\bibitem[Dutton and Brigham(1986)]{DuttonBrigham86}
Ronald~D. Dutton and Robert~C. Brigham.
\newblock Computationally efficient bounds for the catalan numbers.
\newblock \emph{Europ.\ J.\ Comb.}, 7\penalty0 (3):\penalty0 211--213, 1986.

\bibitem[Dwork and Roth(2014)]{DR14}
Cynthia Dwork and Aaron Roth.
\newblock The algorithmic foundations of differential privacy.
\newblock \emph{Foundations and Trends in Theoretical Computer Science}, 9\penalty0 (3-4):\penalty0 211--407, 2014.

\bibitem[Dwork et~al.(2006)Dwork, McSherry, Nissim, and Smith]{DworkMNS06}
Cynthia Dwork, Frank McSherry, Kobbi Nissim, and Adam~D. Smith.
\newblock Calibrating noise to sensitivity in private data analysis.
\newblock In \emph{Proc.\ 3rd Conf.\ Theory of Cryptography (TCC)}, volume 3876, pages 265--284, 2006.

\bibitem[D’Auriac et~al.(2024)D’Auriac, Cholak, Dzhafarov, Monin, and Patey]{2024milliken}
P.E.A. D’Auriac, P.A. Cholak, D.D. Dzhafarov, B.~Monin, and L.~Patey.
\newblock \emph{Milliken’s Tree Theorem and Its Applications: A Computability-Theoretic Perspective}.
\newblock Memoirs of the Amer.\ Math.\ Society. Amer.\ Math.\ Society, 2024.

\bibitem[Erd\H{o}s and Rado(1952)]{ErdosRado52}
P.~Erd\H{o}s and R.~Rado.
\newblock Combinatorial theorems on classifications of subsets of a given set.
\newblock \emph{Proceedings of the London Mathematical Society}, 3\penalty0 (2):\penalty0 417--439, 1952.

\bibitem[Erd\H{o}s and Szekeres(1935)]{ErdosSzekeres35}
Paul Erd\H{o}s and George Szekeres.
\newblock A combinatorial problem in geometry.
\newblock \emph{Compositio Mathematica}, 2:\penalty0 463--470, 1935.

\bibitem[Feldman and Xiao(2015)]{FeldmanX15}
Vitaly Feldman and David Xiao.
\newblock Sample complexity bounds on differentially private learning via communication complexity.
\newblock \emph{SIAM Journal on Computing}, 44\penalty0 (6):\penalty0 1740--1764, 2015.

\bibitem[Francisco et~al.(2015)Francisco, Mermin, and Schweig]{Francisco15}
Christopher Francisco, Jeffrey Mermin, and Jay Schweig.
\newblock Catalan numbers, binary trees, and pointed pseudotriangulations.
\newblock \emph{Europ.\ J.\ Comb.}, 45, 04 2015.

\bibitem[Frankl and Wilson(1981)]{Frankl1981IntersectionTW}
Peter Frankl and Richard~M. Wilson.
\newblock Intersection theorems with geometric consequences.
\newblock \emph{Combinatorica}, 1:\penalty0 357--368, 1981.

\bibitem[Furstenberg(1977)]{Furstenberg1977}
Harry Furstenberg.
\newblock Ergodic behavior of diagonal measures and a theorem of szemer{\'e}di on arithmetic progressions.
\newblock \emph{Journal d’Analyse Math{\'e}matique}, 31:\penalty0 204--256, 1977.

\bibitem[Furstenberg(1981)]{FurstenbergBook}
Harry Furstenberg.
\newblock \emph{Recurrence in Ergodic Theory and Combinatorial Number Theory}.
\newblock Princeton University Press, 1981.

\bibitem[Furstenberg and Weiss(2003)]{FurstenbergWeiss2003}
Hillel Furstenberg and Benjamin Weiss.
\newblock Markov processes and ramsey theory for trees.
\newblock \emph{Combinatorics, Probability and Computing}, 12\penalty0 (5–6):\penalty0 547–563, 2003.

\bibitem[Gasarch(2022)]{WebRamseyCS}
William Gasarch.
\newblock Applications of ramsey theory to computer science.
\newblock \url{https://www.cs.umd.edu/~gasarch/TOPICS/ramsey/ramsey.html}, 2022.
\newblock Accessed: 25/03/2024.

\bibitem[Graham et~al.(1991)Graham, Rothschild, and Spencer]{graham1991ramsey}
R.L. Graham, B.L. Rothschild, and J.H. Spencer.
\newblock \emph{Ramsey Theory}.
\newblock Wiley Series in Discrete Mathematics and Optimization. Wiley, 1991.

\bibitem[Hanneke et~al.(2023)Hanneke, Moran, and Shafer]{hanneke:23}
Steve Hanneke, Shay Moran, and Jonathan Shafer.
\newblock A trichotomy for transductive online learning.
\newblock In \emph{Proc.\ 37th Conf.\ Adv.\ Neural Information Processing Systems (NeurIPS)}, 2023.

\bibitem[Hodges(1997)]{HodgesModelTheory}
Wilfrid Hodges.
\newblock \emph{A shorter model theory}.
\newblock Cambridge University Press, 1997.

\bibitem[Inc.(2024)]{OEIS}
The OEIS~Foundation Inc.
\newblock On-line encyclopedia of integer sequences.
\newblock \url{https://oeis.org/}, 2024.
\newblock Accessed: 16/05/2024.

\bibitem[Jung et~al.(2020)Jung, Kim, and Tewari]{JungKT20}
Young~Hun Jung, Baekjin Kim, and Ambuj Tewari.
\newblock On the equivalence between online and private learnability beyond binary classification.
\newblock In \emph{Proc.\ 33rd Conf.\ Adv.\ Neural Information Processing Systems (NeurIPS)}, 2020.

\bibitem[Kalavasis et~al.(2022)Kalavasis, Velegkas, and Karbasi]{KalavasisVK22}
Alkis Kalavasis, Grigoris Velegkas, and Amin Karbasi.
\newblock Multiclass learnability beyond the {PAC} framework: Universal rates and partial concept classes.
\newblock In \emph{Proc.\ 35th Conf.\ Adv.\ Neural Information Processing Systems (NeurIPS)}, 2022.

\bibitem[Kaplan et~al.(2020)Kaplan, Ligett, Mansour, Naor, and Stemmer]{KaplanLMNS20}
Haim Kaplan, Katrina Ligett, Yishay Mansour, Moni Naor, and Uri Stemmer.
\newblock Privately learning thresholds: Closing the exponential gap.
\newblock In Jacob~D. Abernethy and Shivani Agarwal, editors, \emph{Conference on Learning Theory, {COLT} 2020, 9-12 July 2020, Virtual Event [Graz, Austria]}, volume 125 of \emph{Proceedings of Machine Learning Research}, pages 2263--2285. {PMLR}, 2020.

\bibitem[Kasiviswanathan et~al.(2011)Kasiviswanathan, Lee, Nissim, Raskhodnikova, and Smith]{KasiLNRS11}
Shiva~Prasad Kasiviswanathan, Homin~K. Lee, Kobbi Nissim, Sofya Raskhodnikova, and Adam~D. Smith.
\newblock What can we learn privately?
\newblock \emph{{SIAM} J. Comput.}, 40\penalty0 (3):\penalty0 793--826, 2011.

\bibitem[Littlestone(1987)]{Lit87}
Nick Littlestone.
\newblock Learning quickly when irrelevant attributes abound: A new linear-threshold algorithm.
\newblock In \emph{Proc.\ 28th Symp.\ Foundations of Computer Science (FOCS)}, pages 68--77, 1987.

\bibitem[Littlestone(1988)]{Lit88}
Nick Littlestone.
\newblock Learning quickly when irrelevant attributes abound: A new linear-threshold algorithm.
\newblock \emph{Machine Learning}, 2:\penalty0 285--318, 1988.

\bibitem[Long(2001)]{LongPartial}
Philip Long.
\newblock On agnostic learning with $\{0, *, 1\}$-valued and real-valued hypotheses.
\newblock In \emph{Proc.\ 14th Conf.\ Learning Theory (COLT)}, 2001.

\bibitem[Milliken(1979)]{Milliken79}
Keith~R. Milliken.
\newblock A ramsey theorem for trees.
\newblock \emph{J. Comb.\ Theory, Series A}, 26\penalty0 (3):\penalty0 215--237, 1979.

\bibitem[Pabbaraju(2024)]{chirag:24}
Chirag Pabbaraju.
\newblock Multiclass learnability does not imply sample compression.
\newblock In \emph{Proc.\ 35th Intl.\ Conf.\ Algorithmic Learning Theory (ALT)}, 2024.

\bibitem[Pach et~al.(2012)Pach, Tardos, and Solymosi]{PachTS12}
J{\'{a}}nos Pach, G{\'{a}}bor Tardos, and J{\'{o}}zsef Solymosi.
\newblock Remarks on a ramsey theory for trees.
\newblock \emph{Combinatorica}, 32\penalty0 (4):\penalty0 473--482, 2012.

\bibitem[Ramsey(1930)]{Ramsey1930}
F.~P. Ramsey.
\newblock On a problem of formal logic.
\newblock \emph{Proceedings of the London Mathematical Society}, s2-30:\penalty0 264--286, 1930.

\bibitem[Roberts(1984)]{Roberts84}
Fred~S. Roberts.
\newblock Applications of ramsey theory.
\newblock \emph{Disc.\ Appl.\ Math.}, 9\penalty0 (3):\penalty0 251--261, 1984.

\bibitem[Rosta(2004)]{Rosta2004}
Vera Rosta.
\newblock Ramsey theory applications.
\newblock \emph{Electronic Journal of Combinatorics}, 1000, 2004.

\bibitem[Shalev-Shwartz and Ben-David(2014)]{shalev2014understanding}
Shai Shalev-Shwartz and Shai Ben-David.
\newblock \emph{Understanding machine learning: From theory to algorithms}.
\newblock Cambridge university press, 2014.

\bibitem[Sivakumar et~al.(2021)Sivakumar, Bun, and Gaboardi]{SivakumarBG21}
Satchit Sivakumar, Mark Bun, and Marco Gaboardi.
\newblock Multiclass versus binary differentially private {PAC} learning.
\newblock In \emph{Proc.\ 34th Conf.\ Adv.\ Neural Information Processing Systems (NeurIPS)}, pages 22943--22954, 2021.

\bibitem[Szemer{\'e}di(1975)]{Szemeredi1975}
Endre Szemer{\'e}di.
\newblock On sets of integers containing k elements in arithmetic progression.
\newblock \emph{Acta Arithmetica}, 27:\penalty0 199--245, 1975.

\bibitem[Vadhan(2017)]{Vadhan17survey}
Salil~P. Vadhan.
\newblock The complexity of differential privacy.
\newblock In \emph{Tutorials on the Foundations of Cryptography}, pages 347--450. Springer International Publishing, 2017.

\bibitem[Valiant(1984)]{Valiant84}
Leslie~G. Valiant.
\newblock A theory of the learnable.
\newblock \emph{Commun. {ACM}}, 27\penalty0 (11):\penalty0 1134--1142, 1984.

\bibitem[Wigderson(2019)]{AviWigBook}
Avi Wigderson.
\newblock \emph{Mathematics and Computation: A Theory Revolutionizing Technology and Science}.
\newblock Princeton University Press, 2019.

\end{thebibliography}

    \newpage

    \appendix

\section{Upper Bound for the Ramsey Number for Chains}\label{app:upper_bound_ramsey_chains}

In the proof of \Cref{thm:ramsey_chains} we showed the existence of the Ramsey number $\Ramsey(d,m,k)$. Here we give an explicit calculation for the sufficient depth~$n$ that is required for the proof. Subsequently, we derive an upper bound for~$\Ramsey(d,m,k)$.

The procedure described in the proof of \Cref{thm:ramsey_chains} can be continued until step ${t=\Ramsey(d,m-1,k^2)}$ if for every sequence~$\sigma$ of length~$t$, $\mathtt{depth}(S_{\sigma}) \ge 0$. Consider a sequence~$\sigma$ of length~$t$. To ease the notation, for every step $i\in\{m-2,\ldots,t\}$, $\mathtt{depth}(S_{\sigma(i)})$ is denoted by~$d_i$.
Recall, by \Cref{eq:recursive_eq_inside_proof} the following holds.
\begin{equation}\label{eq:inductive_theOG}
    \left\{
    	\begin{array}{ll}
    d_{m-2} &= n-(m-2),\\
    d_{i+1} &\ge \left\lfloor\frac{d_i -1}{k^{i \choose {m-2}}}\right\rfloor.
    	\end{array}
    \right.
\end{equation}
Observe that if $d_i \ge 2 k^{i \choose {m-2}}$ then
\begin{equation}\label{eq:recurrence_relation}
    d_{i+1}\geq 
\left\lfloor\frac{d_i -1}{k^{i \choose {m-2}}}\right\rfloor \ge \frac{d_i}{k^{i \choose {m-2}}} - 1 \ge \frac{d_i}{2k^{i \choose {m-2}}}.
\end{equation}
If $d_i < 2 k^{i \choose {m-2}}$, then
\[
\left\lfloor\frac{d_i -1}{k^{i \choose {m-2}}}\right\rfloor \in\{0,1\} \;,
\]
meaning that the procedure terminates, or continues for one more last step, therefore we can assume that the bound in \Cref{eq:recurrence_relation} holds in every step $i$.
By induction, using the recurrence relation in \Cref{eq:inductive_theOG,eq:recurrence_relation},
\begin{equation}\label{eq:absolute_bound_on_depth}
    d_i \ge \frac{n-(m-2)}{2^{i-(m-2)}k^{\sum_{j=m-2}^{i-1} \binom{j}{m-2} }} = \frac{n-(m-2)}{2^{i-(m-2)}k^{ \binom{i}{m-1} }} \;,
\end{equation}
because $\sum_{j=m-2}^{i-1} \binom{j}{m-2} = \binom{i}{m-1} $ (the left-hand-side counts the number $(m-1)$-subsets $S\subseteq [i]$, by partitioning them according to the largest element).
For the procedure to continue $t$ steps we require $d_t\geq1$.
Together with \Cref{eq:absolute_bound_on_depth} we deduce the following bound:
\begin{equation*}
    n \ge 2^{t-(m-2)}k^{ \binom{t}{m-1} } + (m-2).
\end{equation*}
Notice that for every $m,k\geq2$ and $t\geq m-2$,
\begin{equation*}
2^{t-(m-2)}k^{ \binom{t}{m-1} } + (m-2) \leq 2^{t-(m-2)}k^{t^{m-1}} + (m-2)\leq k^{2t^{m-1}}.
\end{equation*}

Therefore, choosing $n=k^{2t^{m-1}}$ is sufficient.
Recall that $t=\Ramsey(d,m-1,k^2)$, therefore the following recursive relation is obtained.
\begin{equation}\label{eq:ramsey_recursive_bound}
    \Ramsey(d,m,k)\leq k^{2\Ramsey(d,m-1,k^2)^{m-1}}.
\end{equation}

From now on we will use the Knuth notation $a \uparrow b$ in place of $a^b$ to ease the calculations, and recall that the Knuth's operator is right-associative, i.e.\ $a \uparrow b\uparrow c = a \uparrow (b\uparrow c)$.
By applying \Cref{eq:ramsey_recursive_bound} repeatedly we obtain the following bound
\begin{align*}
\Ramsey(d,m,k) &\le k^2 \uparrow \Ramsey(d,m-1,k^2) \uparrow (m-1)\\
&\leq  k^2 \uparrow (k^{2\cdot 2} \uparrow (m-1)) \uparrow \Ramsey(d,m-2,k^{2^2}) \uparrow (m-2)\\
&\leq  k^2 \uparrow (k^{2\cdot 2} \uparrow (m-1)) \uparrow (k^{2\cdot2^2} \uparrow (m-2)) \uparrow \Ramsey(d,m-3,k^{2^3}) \uparrow (m-3)\\
&\leq \ldots \\
&\leq k^2 \uparrow (k^{2\cdot 2} \uparrow (m-1)) \uparrow (k^{2\cdot2^2} \uparrow (m-2)) \uparrow \ldots \uparrow (k^{2\cdot 2^{m-2}} \uparrow 2) \uparrow \Ramsey(d,1,k^{2^{m-1}})\uparrow 1 \\
&= k^2 \uparrow (k^{2\cdot 2} \uparrow (m-1)) \uparrow (k^{2\cdot 2^2} \uparrow (m-2)) \uparrow \ldots \uparrow (k^{2\cdot 2^{m-2}} \uparrow 2) \uparrow dk^{2^{m-1}},
\end{align*}
where $\Ramsey(d,1,k)=dk$ by the pigeonhole principle (\cref{prop:php}).
Denote 
\begin{equation*}
    R_i= \begin{cases}
			dk^{2^{m-1}}\coloneqq R  & \text{if $i=1$;}\\[10pt]
            k^{2\cdot 2^{m-i}\cdot i \cdot R_{i-1}}
            & \text{if $1<i<m$;}\\[10pt]
            k^{2\cdot R_{m-1}}
            & \text{if $i=m$.}
		 \end{cases}
\end{equation*}
Using this notation, $\Ramsey(d,m,k) \leq R_m$.

\begin{claim}\label{claim:bounding_with_twr}
    For $2\leq i\leq m$,
    \[R_i\leq \twr_{(i)}(c_i \cdot 2^{m-2}R\log k),\]
    where\footnote{We use the convention $\sum_{k=n_1}^{n_2}f(k)=0$ if $n_2<n_1$, hence $c_2=4$.} 
    \[c_i=4+\sum_{j=3}^i{\frac{\max\{1,\log_{(j-2)}(2\cdot 2^{m-j}j\log k)\}}{2^{m-2}R\log k}}.\]
\end{claim}

\begin{proof}[Proof of \Cref{claim:bounding_with_twr}]
    Proof by induction on $i$.
    For $i=2$, 
    \[R_2=k^{4\cdot 2^{m-2} \cdot R}=\twr_{(2)}(c_2\cdot 2^{m-2}R\log k).\]
    For $2<i<m$,
    \begin{align*}
        R_i 
            &= k^{2\cdot 2^{m-i}\cdot i \cdot R_{i-1}}\\
            & =\twr_{(2)}\left[2\cdot 2^{m-i} i\log k \cdot R_{i-1}\right]
            \\
            &\leq \twr_{(2)}\left[2\cdot 2^{m-i}\cdot i\log k\cdot\twr_{(i-1)}(c_{i-1}\cdot 2^{m-2}R\log k)\right] \tag{by induction.}
            \\
            &=\twr_{(2)}\left[\twr_{(i-1)}\log _{(i-2)}(2\cdot 2^{m-i} i\log k)\cdot \twr_{(i-1)}(c_{i-1}\cdot 2^{m-2}R\log k)\right]
            \\
            &\leq\twr_{(2)}\left[\twr_{(i-1)}\left(\max \{1,\log _{(i-2)}(2\cdot 2^{m-i} i\log k)\}\right)\cdot \twr_{(i-1)}\left(c_{i-1}\cdot 2^{m-2}R\log k\right)\right]
            \\
            &\leq \twr_{(2)}\left[\twr_{(i-1)}(\max \{1,\log _{(i-2)}(2\cdot 2^{m-i} i\log k)\})+c_{i-1}\cdot 2^{m-2}R\log k)\right]
            \tag{$\star$}
            \\
            &=\twr_{(2)}\left[\twr_{(i-1)}(c_i \cdot 2^{m-2}R\log k)\right] 
            \tag{definition of $c_i$.}
            \\
            &=\twr_{(i)}(c_i \cdot 2^{m-2}R\log k),
    \end{align*}
    where the inequality marked with $(\star)$ holds since 
    \[\twr_{(n)}(x)\cdot\twr_{(n)}(y)\leq \twr_{(n)}(x\cdot y)\]
    for $x,y\geq1, n\geq2$.
    
    The case $i=m$ follows because $R_m=k^{2\cdot R_{m-1}}\leq k^{2m\cdot R_{m-1}}$, and $k^{2m\cdot R_{m-1}}\leq \twr_{(m)}(c_m \cdot 2^{m-2}R\log k)$, by using the above calculation for $2<i<m$ one more time for $i=m$.
    
\end{proof}

\begin{corollary}[Upper Bound For Ramsey Number for Chains]\label{cor:upper_bound_for_Ramsey_number}
    For every integers $m\geq 2,d\geq m,k\geq 2$,
    \begin{equation*}
        \Ramsey(d,m,k)\leq \twr_{(m)}(5\cdot 2^{m-2}dk^{2^{m-1}}\log k).
    \end{equation*}
\end{corollary}

\begin{proof}
     It suffices to bound
      \[c_m= 4+\sum_{j=3}^m{\frac{\max\{1,\log_{(j-2)}(2\cdot 2^{m-j}j\log k)\}}{2^{m-2}R\log k}}.\]  
      Note that for every $3\leq j\leq m$, 
      \[\log_{(j-2)}(2\cdot 2^{m-j}j\log k)\leq 2\cdot 2^{m-j}\log k\leq 2\cdot 2^{m-2}\log k.\]
      Therefore,
      \[c_m\leq 4+ \sum_{j=3}^m{\frac{2\cdot 2^{m-2}\log k}{2^{m-2}R\log k}}= 4+\frac{2(m-2)}{dk^{2^{m-1}}}\leq 5.\]
\end{proof}

\section{More About Types}\label{app:types}
\Cref{thm:ramsey_general_finite} states that for every coloring of $m$-subsets of a complete binary tree, there exists a deep subtree that is colored with at most $\tau(m)$ colors, where $\tau(m)$ is the number of possible types of $m$-subsets. As we mentioned in \cref{sec:main_results:ramsey}, $\tau(m)$ is optimal in the sense that if we color $m$-subsets according to their type, every subtree of depth at least $m-1$ must contain $m$-subsets of all possible types, and therefore cannot see less then $\tau(m)$ colors. It is interesting to note that the optimal number of colors $\tau(m)$ depends only on the size of the colored subsets $m$, rather than on $d$ or $k$, where $d$ is the desired depth of the subtree and $k$ is the number of colors used in the coloring.

We turn to calculate $\tau(m)$. Recall that a set of vertices $A$ is \emph{closed} if for all $u,v\in A$ also $\LCA(u,v)\in A$, and the \emph{closure} of $A$, denoted $\bar A$, is the minimal closed set containing $A$. Two sets $A_1, A_2$ are equivalent if there is an isomorphism $\varphi:\bar {A_1}\to \bar {A_2}$ that respects being left/ right descendant relation, and also $\varphi(A_1)=A_2$.  
The \emph{types} of $m$-subsets are the induced equivalence classes.  
(See \cref{def:closure,def:type_subset}.)

\begin{proposition}[Number of Possible Types]\label{prop:num_of_types}
    The number of possible types for $m$-subsets $\tau(m)$ corresponds to the \emph{generalized Catalan number} denoted as $C(2,m)$ in the On-line Encyclopedia of Integer Sequences (OEIS)~\cite{OEIS} (A064062), for all $m \ge 1$. In particular, it holds that:
    \begin{equation*}
        \tau(m) = \sum_{k = 0}^{m-1} (m-k) \binom{m+k-1}{k} \frac{2^k}{m}, \; \text{ for all } m \ge 1.
    \end{equation*}
\end{proposition}
\begin{proof}
Let us first introduce the notion of \emph{branch-marked} binary trees (as defined by~\citet*{Francisco15}). Let $T$ be any binary tree\footnote{The definition we consider for a binary tree is a directed tree of which every vertex has at most one left child and at most one right child.} 
and $v$ one of its vertices: we call $v$ a \emph{branching vertex} if it has two children. If $B$ is any set of branching vertices of $T$, we call the pair $(T,B)$ a \emph{branch-marked} binary tree and \emph{unmarked} all the vertices that are not in $B$. Now, ~\citet*{CountingBorel19} showed that the number of all possible branch-marked binary trees with $m$ unmarked vertices is exactly equal to $C(2, m)$. As a consequence, it is sufficient to prove that there exists a bijection between the possible types of $m$-subsets and such trees, to conclude the proof.
In the following, we will show such a bijection.

First, let $A$ be an $m$-subset of a complete binary tree $T$ of depth $\ge m$ and let $\bar{A}$ be its closure. Denote by $T_A$ the binary tree that is isomorphic to $\bar A$ as a subtree of $T$; i.e.\ there is a bijection~$\varphi$ between $\bar A$ and the vertices of $T_A$ that preserves the being left/right descendant relation. Note that by the definition of type, $ A'$ has the same type as $A$ if and only if $T_A=T_{A'}$ and also $A$ and $A'$ are mapped to the same set of vertices in $T_A$/ $T_{A'}$.

We claim that the pair $(T_A,\varphi(\bar A\setminus A))$ is a branch-marked binary tree with $m$ unmarked vertices. In other words, each vertex $v \in \bar{A} \setminus A$ corresponds to a branching vertex in $T_{A}$. Indeed, if $\varphi(v)$ were not a branching vertex then it would have either no children or only one child, and in both cases $\bar A\setminus\{v\}$ would be closed and containing $A$, which is a contradiction to the minimality of $\bar{A}$. 

Denote by $\Phi$ the map that sends the type of a subset $A$ to the branched-marked binary tree $(T_A,\varphi(\bar A\setminus A))$. By the properties discussed above, $\Phi$ is well-defined and injective. 
It remains to show that $\Phi$ is surjective. Given a branch-marked tree $(S,B)$ with $m$-unmarked vertices, embed $S$ in a complete binary tree. We claim that the type of the $m$-subset containing all of the unmarked vertices is mapped by $\Phi$ to $(S,B)$.
Denote by $A$ the $m$-subset of the unmarked vertices. It suffices to show that $\bar A=A\cup B$. 
It is clear that $\bar A\subseteq A\cup B$ since $S$ is a tree (and therefore $A\cup B$ is closed). On the other hand, if there exists $v\in B\setminus \bar A$, pick the deepest such $v$. Since $v$ is a branching vertex it is also the $\LCA$ of its two children, which are vertices in $\bar A$ by the choice of $v$, and therefore must be itself in $\bar A$, a contradiction.
\end{proof}

\begin{corollary}
    The number of possible types for $m$-subsets $\tau(m)$ satisfies $\tau(m)\leq \frac{2^{3m-2}}{\sqrt{\pi (m-1)}}$, for all $m\geq  1$.
\end{corollary}
\begin{proof}
We claim that the following chain of inequalities holds:
  \begin{align*}
    \tau(m) &= \sum_{k = 0}^{m-1} (m-k) \binom{m+k-1}{k} \frac{2^k}{m}\\
    &\le \binom{2m-1}{m-1} 2^{m-1} = \binom{2m-2}{m-1} \frac{(2m-1)2^{m-1}}{m} \\
    &= C_{m-1} 2^{m-1} (2m-1)\\
    &\le \frac{4^{m-1} \cdot 2^{m-1} \cdot (2m-1)}{m\sqrt{\pi (m-1)}} = \frac{2^{3m-3} (2m-1)}{m\sqrt{\pi (m-1)}}\\
    &\le \frac{2^{3m-2}}{\sqrt{\pi (m-1)}},\\
\end{align*}
where $C_{m-1}$ is the $(m-1)$-th Catalan number. The first inequality holds by \Cref{prop:num_of_types}, if we upper bound every $2^k$ with $2^{m-1}$ and use the binomial equality:
\[
\sum_{k=0}^{n} \binom{m+k-1}{k} = \binom{m+n}{n},
\]
with $n=m-1$. The second-to-last inequality instead, where the right-hand quantity corresponds also to the asymptotical limits of Catalan numbers,  was proven by~\citet*{DuttonBrigham86} for $m \ge 1$.  
\end{proof}

\section{Ramsey for Infinite Trees}\label{app:ramsey_infinite}

In this section, we discuss the case of coloring $m$-subsets of an infinite complete binary tree. By compactness, \cref{thm:ramsey_general_finite} implies that for every such coloring, there exists a $\tau(m)$-chromatic subtree of arbitrarily large finite depth. However, we claim that this is not true that there exists an \ul{infinite} $\tau(m)$-chromatic subtree. 

\medskip

For simplicity, we start by describing a coloring of pairs in an infinite tree that does not admit an infinite trichromatic subtree.
Given two incomparable vertices $u,v$, we say that $u$ is the left vertex and $v$ is the right vertex if $u$ is a left descendant of $\LCA(u,v)$, and $v$ is a right descendant of $\LCA(u,v)$.
Consider the following pairs-coloring of an infinite complete binary tree $T$:
\begin{itemize}
    \item if $\{u,v\}$ is a left-pair then $\chi(\{u,v\})=\mathsf{red}$;
    \item if $\{u,v\}$ is a right-pair then $\chi(\{u,v\})=\mathsf{blue}$;
    \item if  $\{u,v\}$ are incomparable, and the depth of the left vertex is at least the depth of the right vertex, then $\chi(\{u,v\})=\mathsf{green}$;
    \item if  $\{u,v\}$ are incomparable, and the depth of the left vertex is smaller than the depth of the right vertex, then $\chi(\{u,v\})=\mathsf{yellow}$.
\end{itemize}

Indeed, every infinite subtree admits all four colors. Suppose by contradiction that there exists an infinite trichromatic subtree $T'$, and without loss of generality assume that it is colored by $\{\mathsf{red, blue, green}\}$.
Take a vertex~$u$ in the left subtree of $T'$ and denote by $d$ its depth in the original tree $T$. 
Since $T'$ is infinite, there is a vertex $v$ in the right subtree of $T'$ with a depth in $T'$ that is larger than $d$, hence its depth in $T$ is also larger than $d$. By the definition of the coloring, we must have $\chi(\{u,v\})=\mathsf{yellow}$ which is a contradiction.

\medskip

This example can be generalized for $m\geq 2$: color all $m$-chains according to their type, and color all $m$-anti-chains (i.e.\ $m$-subsets such that every two vertices are incomparable) according to the permutation of their depths when ordering the vertices from left to right.

\begin{corollary}
    The minimum number $\tau_{\infty}(m)$ needed to guarantee the existence of an infinite $\tau_{\infty}(m)$~-~chromatic subtree for any given coloring of $m$-subsets, satisfies $\tau_{\infty}(m)\geq m !+2^{m-1}$. In particular, $\tau_{\infty}(m)=\omega(\tau(m))$.
\end{corollary}

\section{Additional Proofs}\label{app:additional_proofs}

\subsection{Proof of the Pigeonhole Principle for Trees}\label{app:additional_proofs:php}

We prove here a more general version of \Cref{item:php} of \cref{prop:php}.

\begin{proposition}[Pigeonhole Principle on Trees]\label{prop:php_general}
	 Let $T$ be a complete binary tree of depth $n$, and let $k\in \bbN$. Then for every vertex coloring of $T$ with $k$ colors $c_1,\ldots,c_k$, and every $a_1,\ldots, a_k\in\bbN$ such that $a_1+\ldots+a_k=n$, there exists a $c_i$-monochromatic subtree of depth $a_i$ for some $i\in[k]$.   
\end{proposition}

Indeed, \Cref{item:php}  follows from \Cref{prop:php_general} by setting each $a_i$ to be at least $d$, which is the desired depth for the monochromatic subtree in the statement of \Cref{prop:php}.

\begin{proof}
	We first prove the statement for $k=2$. We use induction on the depth $d$ of the tree. The base case $d=0$ is trivial. Consider an arbitrary coloring of a tree $T$ of depth $d$ with two colors,~$c_1,c_2$. Assume without loss of generality that the color of the root $v_{root}$ is $c_1$.
	Let $T_0,T_1$ be the left and right subtrees of $v_{root}$, respectively.
	If there exists a $c_2$-monochromatic subtree of depth~$a_2$ of either $T_0$ or $T_1$, then we are done. If not, since $(a_1-1)+a_2=d-1$, by induction there exist $c_1$-monochromatic subtrees $T_0'$ and $ T_1'$ of $T_0$ and $T_1$ of depth $a_1-1$. $v_{root}$ together with $T_0'$ and $T_1'$ yield a $c_1$-monochromatic subtree of $T$ of depth $a_1$.
	
 For $k>2$, induct on the number of colors $k$. Consider a coloring of $T$ with $k$ colors $c_1,\ldots,c_k$, and let $a_1+\ldots +a_k=n$. Define a new coloring using $k-1$ colors by considering $c_{k-1}$ and $c_k$ as the same new color $c'$. By induction, either there exists a $c_i$-subtree of depth $a_i$ for $1\leq i\leq k-2$,  or there exists a $c'$-subtree of depth $a_{k-1}+a_k$. If the latter holds, then by the base case, when considering again $c_{k-1}$ and $c_k$ as distinct colors, we are done.
\end{proof}

\subsection{Proof of \Cref{lem:counting_subtrees}}\label{app:additional_proofs:counting_subtrees}

\begin{lemma*}[Restatement of \Cref{lem:counting_subtrees}]
Let $T$ be a complete binary tree of depth $n$. Then the number of its level-aligned subtrees of depth $d$ is upper bounded by $\binom{n}{d+1} \cdot 2^{n2^d}$.
\end{lemma*}

\begin{proof}
    By hypothesis, the subtrees are level-aligned i.e.\ the vertices at the same distance from the root lie at the same level in the original tree. For a fixed subtree $S$, we consider its levels to be labeled by numbers between $0$, corresponding to the root of $S$, and $d$ corresponding to its leaves. Let $l_i$ (for $i=0,\ldots, d$) be the level in $T$ corresponding to level $i$ in $S$ and let $a_i = l_i - l_{i-1}$ (for $i=1,\ldots, d$).

    For a fixed sequence $\{l_i\}_{i=0}^{d}$, we first upper bound the number of possible subtrees $S$ that respect this choice of levels. Observe the following: for any vertex $v$ at a level $i$ the number of possible choices for children at level $i+1$ is exactly $2^{2(a_{i+1}-1)}$. Indeed, the number of left descendants of $v$ at level $l_{i+1}$ is $2^{a_{i+1}} / 2$, and the same goes for the right descendants, which implies that the number of possible couples is $2^{2(a_{i+1}-1)}$. The number of possible choices for a root instead is $2^{l_0}$. The number of possible subtrees respecting the choice of levels is then:
    
    \begin{align*}
    2^{l_0} \cdot \prod_{i=1}^{d} 2^{2(a_i -1) 2^{i-1}} &=\\
    2^{l_0 + \sum_{i=1}^{d} 2^i(a_i -1)} &\le\\
    2^{l_0 + 2^d\sum_{i=1}^{d}(a_i -1)} &\le\\
    2^{l_0 + 2^d(n - d - l_0)} &\le 2^{n2^d}
    \end{align*}
    
    where we employed the fact that $l_0 + \sum_{i=1}^{d} a_i \le n$. Given that the number of possible choices of levels is $\binom{n}{d+1}$, we obtain the desired result.
\end{proof}

\end{document}